\documentclass{article}

\PassOptionsToPackage{numbers, sort&compress}{natbib}

\usepackage[final]{neurips_2021}

\newif\iffinal
\finaltrue

\usepackage[utf8]{inputenc} %
\usepackage[T1]{fontenc}    %
\usepackage{hyperref}       %
\usepackage{url}            %
\usepackage{booktabs}       %
\usepackage{amsfonts}       %
\usepackage{nicefrac}       %
\usepackage{microtype}      %
\usepackage{xcolor}         %
\usepackage{amssymb}
\usepackage{amsmath}
\usepackage{cancel}
\usepackage{bbm}
\usepackage{amsthm}
\usepackage{wrapfig}
\usepackage{algorithm}
\usepackage[noend]{algpseudocode}
\usepackage{subcaption}
\usepackage{enumitem}
\usepackage{tikz}
\usetikzlibrary{calc}

\newcommand{\tikzmark}[1]{\tikz[overlay,remember picture] \node (#1) {};}

\newcommand{\figleft}{{\em (Left)}}

\newcommand{\figright}{{\em (Right)}}
\newcommand{\figtop}{{\em (Top)}}
\newcommand{\figbottom}{{\em (Bottom)}}

\def\eqref#1{equation~\ref{#1}}

\def\1{\bm{1}}

\DeclareMathAlphabet{\mathsfit}{\encodingdefault}{\sfdefault}{m}{sl}
\SetMathAlphabet{\mathsfit}{bold}{\encodingdefault}{\sfdefault}{bx}{n}

\def\gD{{\mathcal{D}}}

\def\gL{{\mathcal{L}}}

\def\gP{{\mathcal{P}}}

\def\gS{{\mathcal{S}}}

\newcommand{\E}{\mathbb{E}}

\DeclareMathOperator*{\argmax}{arg\,max}

\newcommand{\CE}{\mathcal{CE}}
\newcommand{\cs}{\mathbf{s_t}}  %
\newcommand{\ca}{\mathbf{a_t}}  %
\newcommand{\ns}{\mathbf{s_{t+1}}}  %
\newcommand{\na}{\mathbf{a_{t+1}}}  %
\newcommand{\fs}{\mathbf{s_{t+}}}  %
\newcommand{\e}{\mathbf{e_t}}  %
\newcommand{\nee}{\mathbf{e_{t+1}}}  %
\newcommand{\fe}{\mathbf{e_{t+}}}  %

\newtheorem{theorem}{Theorem}[section]
\newtheorem{corollary}{Corollary}[theorem]
\newtheorem{lemma}[theorem]{Lemma}

\newtheorem{definition}{Definition}

\usepackage{mathtools}

\DeclarePairedDelimiterX{\infdivx}[2]{(}{)}{%
  #1\;\delimsize\|\;#2%
}

\title{Replacing Rewards with Examples: Example-Based Policy Search via Recursive Classification}

\author{%
  Benjamin Eysenbach$^{1\,2}$ \qquad Sergey Levine$^{2\,3}$ \qquad Ruslan Salakhutdinov$^{1}$ \\
  $^{1}$Carnegie Mellon University, \quad $^{2}$Google Brain, \quad $^{3}$UC Berkeley \\
  \texttt{beysenba@cs.cmu.edu}
}

\begin{document}

\maketitle

\begin{abstract}
Reinforcement learning (RL) algorithms assume that users specify tasks by manually writing down a reward function. However, this process can be laborious and demands considerable technical expertise. Can we devise RL algorithms that instead enable users to specify tasks simply by providing examples of successful outcomes? In this paper, we derive a control algorithm that maximizes the future probability of these successful outcome examples. Prior work has approached similar problems with a two-stage process, first learning a reward function and then optimizing this reward function using another RL algorithm. In contrast, our method directly learns a value function from transitions and successful outcomes, without learning this intermediate reward function. Our method therefore requires fewer hyperparameters to tune and lines of code to debug. We show that our method satisfies a new data-driven Bellman equation, where examples take the place of the typical reward function term. Experiments show that our approach outperforms prior methods that learn explicit reward functions.\footnote{Project site with videos and code:
\iffinal
    \url{https://ben-eysenbach.github.io/rce}
\else
    \url{https://rce-anonymous.github.io/}
\fi
}
\end{abstract}

\vspace{-0.2em}
\section{Introduction}
\vspace{-0.2em}
\label{sec:intro}

In supervised learning settings, tasks are defined by data: what causes a car detector to detect cars is not the choice of loss function (which might be the same as for an airplane detector), but the choice of training data. Defining tasks in terms of data, rather than specialized loss functions, arguably makes it easier to apply machine learning algorithms to new domains.
In contrast, reinforcement learning (RL) problems are typically posed in terms of reward functions, which are typically manually designed. 
Arguably, the challenge of designing reward functions has limited RL to applications with simple reward functions, and has been restricted to users who speak this language of mathematically-defined reward functions.
Can we make task specification in RL similarly data-driven?

Whereas the standard MDP formalism centers around predicting and maximizing the future reward, we will instead focus on the problem classifying whether a task will be solved in the future.
The user will provide a collection of example success states, not a reward function. We call this problem setting \emph{example-based control}. In effect, these examples tell the agent ``What would the world look like if the task were solved?" For example, for the task of opening a door, success examples correspond to different observations of the world when the door is open.
The user can find examples of success even for tasks that they themselves do not know how to solve. For example, the user could solve the task using actions unavailable to the agent (e.g., the user may have two arms, but a robotic agent may have only one) or the user could find success examples by searching the internet.
As we will discuss in Sec.~\ref{sec:problem}, this problem setting is different from imitation learning: we maximize a different objective function and only assume access to success examples, not entire expert trajectories.

\begin{figure}[t]
\centering
\vspace{-2.5em}
\includegraphics[width=0.9\linewidth]{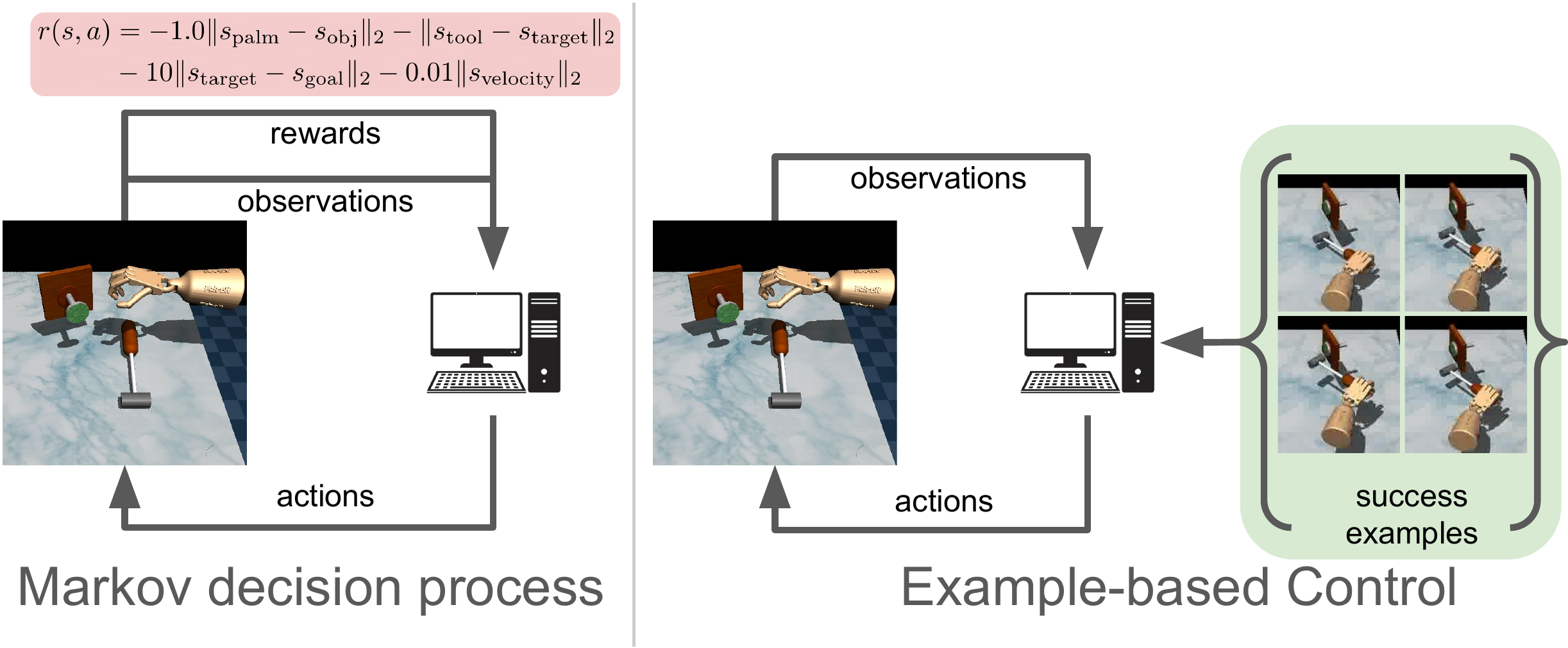}
\vspace{-1.em}
\caption{{\footnotesize \textbf{Example-based control}: Whereas the standard MDP framework requires a user-defined reward function, example-based control specifies tasks via a handful of user-provided success examples.}}
\vspace{-1.5em}
\label{fig:mdp}
\end{figure}

Learning from examples is challenging because we must automatically identify when the agent has solved the task and reward it for doing so.
Prior methods (either imitation learning from demonstrations or learning from success examples) take an indirect approach that resembles inverse RL: first learn a separate model to represent the reward function, and then optimize this reward function with standard RL algorithms.
Our method is different from these prior methods because it learns to predict \emph{future} success directly from transitions and success examples, without learning a separate reward function.
This key difference has important algorithmic, theoretical, and empirical benefits.
\textbf{Algorithmically}, our end-to-end approach removes potential biases in learning a separate reward function, reduces the number of hyperparameters, and simplifies the resulting implementation.
\textbf{Theoretically}, we propose a method for classifying \emph{future} events using a variant of temporal difference learning that we call recursive classification. This method satisfies a new Bellman equation, where success examples are used in place of the standard reward function term. We use this result to provide convergence guarantees.
\textbf{Empirically}, we demonstrate that our method solves many complex manipulation tasks that prior methods fail to solve.

Our paper also addresses a subtle but important ambiguity in formulating example-based control.
Some states might always solve the task while other states might rarely solve the task. But, without knowing how often the user visited each state, we cannot determine the likelihood that each state solves the task. Thus, an agent can only estimate the probability of success by making an additional assumption about how the success examples were generated. We will discuss two choices of assumptions. The first choice of assumption is convenient from an algorithmic perspective, but is sometimes violated in practice. A second choice is a worst-case approach, resluting in a problem setting that we call \emph{robust example-based control}. Our analysis shows that the robust example-based control objective is equivalent to minimizing the squared Hellinger distance (an $f$-divergence).

In summary, this paper studies a data-driven framing of control, where reward functions are replaced by examples of successful outcomes. %
Our main contribution is an algorithm for off-policy example-based control. The key idea of the algorithm is to directly learn to predict whether the task will be solved in the future via \emph{recursive classification}, without using separate reward learning and policy search procedures. Our analysis shows that our method satisfies a new Bellman equation where rewards are replaced by data (examples of success). Empirically, our method significantly outperforms state-of-the-art imitation learning methods (AIRL~\citep{fu2017learning}, DAC~\citep{kostrikov2018discriminator}, and SQIL~\citep{reddy2020sqil}) and recent methods that learn reward functions (ORIL~\citep{zolna2020offline}, PURL~\citep{xu2019positive}, and VICE~\citep{fu2018variational}).
Our method completes complex tasks, such as picking up a hammer to knock a nail into a board, tasks that \emph{none} of these baselines can solve.
Using tasks with image observations, we demonstrate agents learned with our method acquire a notion of success that generalizes to new environments with varying shapes and goal locations.

\vspace{-0.2em}
\section{Related Work}
\vspace{-0.2em}
\label{sec:related-work}

\paragraph{Learning reward functions.}
Prior works have studied RL in settings where the task is specified either with examples of successful outcomes or complete demonstrations. These prior methods typically learn a reward function from data and then apply RL to this reward function (e.g.,~\citet{ziebart2008maximum, fu2018variational}).
Most inverse RL algorithms adopt this approach~\citep{pomerleau1989alvinn, abbeel2004apprenticeship, ratliff2006maximum, ziebart2008maximum, ross2011reduction, wulfmeier2015maximum}, as do more recent methods that learn a \emph{success classifier} to distinguishing successful outcomes from random states~\citep{fu2018variational, singh2019end, zolna2020offline, konyushkova2020semi}.
Prior adversarial imitation learning methods~\citep{ho2016, fu2017learning} can be viewed as \emph{iteratively} learning a success classifier.
Recent work in this area focuses on extending these methods to the offline setting~\citep{zolna2020offline, konyushkova2020semi}, incorporating additional sources of supervision~\citep{zolna2019task}, and learning the classifier via positive-unlabeled classification~\citep{xu2019positive, irpan2019off, zolna2020offline}. Many prior methods for robot learning have likewise used a classifier to distinguish success examples~\citep{calandra2017feeling, xie2018few, vecerik2019practical, lu2020multi}.
Unlike these prior methods, our approach only requires examples of successful outcomes (not expert trajectories) and does not learn a separate reward function.
Instead, our method learns a value function directly from examples,
effectively ``cutting out the middleman.''
This difference from prior work removes hyperparameters and potential bugs associated with learning a success classifier.
Empirically, we demonstrate that our end-to-end approach outperforms these prior two-stage approaches. See Appendix~\ref{appendix:connections} for more discussion of the relationship between our method and prior work.

\paragraph{Imitation learning without auxiliary classifiers.}
While example-based control is different from imitation learning, our method is similar to two prior imitation learning methods that likewise avoid learning a separate reward function~\citep{kostrikov2019imitation, reddy2020sqil}.
ValueDICE~\citep{kostrikov2019imitation}, a method based on convex duality, uses full expert demonstrations for imitation learning. In contrast, our method learns from success examples, which are typically easier to provide than full expert demonstrations.
SQIL~\citep{reddy2020sqil} is a modification of SAC~\citep{haarnoja2018soft} that labels success examples with a reward of $+1$. 
The mechanics of our method are similar to SQIL~\citep{reddy2020sqil}, but key algorithmic differences (backed by stronger theoretical guarantees) result in better empirical performance.
Our analysis in Sec.~\ref{sec:robust} highlights connections and differences between imitation learning and example-based~control.

\paragraph{Goal-conditioned RL.}
Goal-conditioned RL provides one way to specify tasks in terms of data, and prior work has shown how goal-conditioned policies can be learned directly from data, without a reward function~\citep{kaelbling1993learning, schaul2015universal, lin2019reinforcement, eysenbach2020c}.
However, the problem that we study in this paper, example-based control, is different from goal-conditioned RL because it allows users to indicate that tasks can be solved in many ways, enabling the agent to learn a more general notion of success.
Perhaps the most similar prior work in this area is C-learning~\citep{eysenbach2020c}, which uses a temporal-difference update to learn the probability that a goal state will be reached in the future.
Our method will use a similar temporal-difference update to predict whether \emph{any} success example will be reached in the future. Despite the high-level similarity with C-learning, our algorithm will be markedly different; for example, our method does not require hindsight relabeling, and learns a single policy rather than a goal-conditioned policy.

\vspace{-0.2em}
\section{Example-Based Control via Recursive Classification}
\vspace{-0.2em}

We aim to learn a policy that reaches states that are likely to solve the task (see Fig.~\ref{fig:mdp}), without relying on a reward function. We start by formally describing this problem, which we will call \emph{example-based control}. We then propose a method for solving this problem and provide convergence guarantees.

\vspace{-0.2em}
\subsection{Problem Statement}
\vspace{-0.2em}
\label{sec:problem}

Example-based control is defined by a controlled Markov process (i.e., an MDP without a reward function) with dynamics $p(\ns \mid \cs, \ca)$ and an initial state distribution $p_1(\mathbf{s_1})$, where $\cs \in \gS$ and $\ca$ denote the time-indexed states and actions. The variable $\mathbf{s_{t + \Delta}}$ denotes a state $\Delta$ steps in the future.

The agent is given a set of \emph{success examples}, $\gS^* = \{\mathbf{s^*}\} \subseteq \gS$. The random variable $\e \in \{0, 1\}$ indicates whether the task is solved at time $t$, and $p(\e \mid \cs)$ denotes the probability that the current state $\cs$ solves the task. Given a policy $\pi_\phi(\ca \mid \cs)$,
we define the discounted future state distribution:
\begin{equation}
    p^\pi(\fs \mid \cs, \ca) \triangleq (1 - \gamma) \sum_{\Delta = 0}^\infty \gamma^\Delta p^\pi(\mathbf{s_{t + \Delta}} = \fs \mid \cs, \ca).
\end{equation}
Using this definition, we can write the probability of solving the task at a \emph{future} step as
\begin{equation}
    p^\pi(\fe \mid \cs, \ca) \triangleq \E_{p^\pi(\fs \mid \cs, \ca)}[p(\fe \mid \fs)]. \label{eq:future-prob}
\end{equation}
Example-based control maximizes the probability of solving the task in the (discounted) future:
\begin{definition}[Example-based control] \label{def:event-based-control}
Given a controlled Markov process and distribution over success examples $p(\cs \mid \e = 1)$, the example-based control problem is to find the policy that optimizes the likelihood of solving the task:
\begin{equation*}
    \argmax_\pi p^\pi(\fe = 1) = \E_{p_1(\mathbf{s_1}), \pi(\mathbf{a_1} \mid \mathbf{s_1})}\left[p^\pi(\fe = 1\mid \mathbf{s_1}, \mathbf{a_1}) \right].
\end{equation*}
\end{definition}
Although this objective is equivalent to the RL objective with rewards $r(\cs, \ca) = p(\e =1 \mid \cs)$,
we assume the probabilities $p(\e \mid \cs)$ are unknown. Instead, we assume that we have samples of successful states, $\mathbf{s^*} \sim p_U(\cs \mid \e = 1)$.
Example-based control differs from imitation learning because imitation learning requires full expert demonstrations. In the special case where the user provides a single success state, example-based control is equivalent to goal-conditioned RL.

Since interacting with the environment to collect experience is expensive in many settings, we define \emph{off-policy example-based control} as the version of this problem where the agent learns from environment interactions collected from other policies. In this setting, the agent learns from two distinct datasets: \textbf{(1)} transitions, $\{(\cs, \ca, \ns) \sim p_U(\cs, \ca, \ns)\}$, which contain information about the environment dynamics; and \textbf{(2)} success examples, $\gS^* = \{\mathbf{s^*} \sim p_U(\cs \mid \e = 1)\}$, which specify the task that the agent should attempt to solve. Our analysis will assume that these two datasets are fixed. The main contribution of this paper is an algorithm for off-policy example-based~control. 

\paragraph{An assumption on success examples.}
The probability of solving the task at state $\cs$, $p(\e = 1 \mid \cs)$, cannot be uniquely determined from success examples and transitions alone. To explain this ambiguity, we define $p_U(\cs)$ as the state distribution visited by the user; note that the user may be quite bad at solving the task themselves. Then, the probability of solving the task at state $\cs$ depends on how often a user visits state $\cs$ versus how often the task is solved when visiting state $\cs$:
\begin{equation}
    p(\e = 1 \mid \cs) = \frac{p_U(\cs \mid \e = 1)}{p_U(\cs)}p_U(\e = 1). \label{eq:bayes}
\end{equation}
For example, the user may complete a task using two strategies, but we cannot determine which of these strategies is more likely to succeed unless we know how often the user attempted each strategy. Thus, any method that learns from success examples \emph{must} make an additional assumption on $p_U(\cs)$. We will discuss two choices of assumptions. The \textbf{first choice} is to assume that the user visited states with the same frequency that they occur in the dataset of transitions. That is,
\begin{equation}
    p_U(\cs) = \iint p_U(\cs, \ca, \ns) d\ca d \ns. \label{eq:assumption}
\end{equation}
Intuitively, this assumption implies that the user has the same capabilities as the agent. Prior work makes this same assumption without stating it explicitly~\citep{fu2018variational, singh2019end, nasiriany2020disco}. Experimentally, we find that our method succeeds even in cases where this assumption is violated.

However, many common settings violate this assumption, especially when the user has different dynamics constraints than the agent. For example, a human user collecting success examples for a cleaning task might usually put away objects on a shelf at eye-level, whereas transitions collected by a robot interact with the ground-level shelves more frequently. Under our previous assumption, the robot would assume that putting objects away on higher shelves is more satisfactory than putting them away on lower shelves, even though doing so might be much more challenging for the robot.
To handle these difference in capabilities, the \textbf{second choice} is to use a \emph{worst-case} formulation, which optimizes the policy to be robust to \emph{any} choice of $p_U(\cs)$. Surprisingly, this setting admits a tractable solution, as we discuss in Sec.~\ref{sec:robust}.

\subsection{Predicting Future Success by Recursive Classification}
\label{sec:method}

We now describe our method for example-based control. We start with the more standard \textbf{first choice} for the assumption on $p_U(\cs)$ (Eq.~\ref{eq:assumption}); we discuss the second choice in Sec.~\ref{sec:robust}.
Our approach estimates the probability in Eq.~\ref{eq:future-prob} indirectly via a future success classifier. This classifier, $C_\theta^\pi(\cs, \ca)$, discriminates between ``positive'' state-action pairs which lead to successful outcomes (i.e., sampled from $p(\cs, \ca \mid \fe = 1)$) and random ``negatives'' (i.e., sampled from the marginal distribution $p(\cs, \ca)$).
We will use different class-specific weights, using a weight of $p(\fe = 1)$ for the ``positives'' and a weight of $1$ for the ``negatives.'' Bayes-optimal classifier is
\begin{equation}
C_\theta^\pi(\cs, \ca) = \frac{p^\pi(\cs, \ca \mid \fe = 1) p(\fe = 1)}{p^\pi(\cs, \ca \mid\fe = 1) p(\fe = 1) + p(\cs, \ca)}. \label{eq:bayes-opt}
\end{equation}
These class specific weights let us predict the probability of future success using the optimal classifier:
\begin{align}
    \frac{C_\theta^\pi(\cs, \ca)}{1 - C_\theta^\pi(\cs, \ca)} &= p^\pi(\fe = 1 \mid \cs, \ca). \label{eq:posterior}
\end{align}
Importantly, the resulting method will not actually require estimating the weight $p(\fe = 1)$.
We would like to optimize the classifier parameters using maximum likelihood estimation:
\begin{align}
    \gL^\pi (\theta) \triangleq \; & p(\fe = 1) \; \E_{p(\cs, \ca \mid \fe = 1)} [\log C_\theta^\pi(\cs, \ca)] + \E_{p(\cs, \ca)} [\log (1 - C_\theta^\pi(\cs, \ca))]. \label{eq:obj-1}
\end{align}
However, we cannot directly optimize this objective because we cannot sample from \mbox{$p(\cs, \ca \mid \fe = 1)$}. We convert Eq.~\ref{eq:obj-1} into an equivalent loss function that we can optimize using three steps; see Appendix~\ref{appendix:derivation} for a detailed derivation. The \textbf{first step} is to factor the distribution $p(\cs, \ca, \fe = 1)$. The \textbf{second step} is to decompose $p^\pi(\fe = 1 \mid \cs, \ca)$ into two terms, corresponding to the probabilities of solving the task at time $t' = t+1$ and time $t' > t + 1$. We can estimate the probability of solving the task at the next time step using the set of success examples. The \textbf{third step} is to estimate the probability of solving the task at time $t' > t + 1$ by evaluating the classifier at the next time step. Combining these three steps, we can \emph{equivalently} express the objective function in Eq.~\ref{eq:obj-1} using off-policy data:
\footnotesize \begin{align}
    \gL^\pi(\theta) = & (1 - \gamma) \E_{\substack{p_U(\cs \mid \e = 1) \\p(\ca \mid \cs)}}[\underbrace{\log C_\theta^\pi(\cs, \ca)}_{(a)}] + \E_{p(\cs, \ca, \ns)} [\underbrace{\gamma w \log C_\theta^\pi(\cs, \ca)}_{(b)} + \underbrace{\log (1 - C_\theta^\pi(\cs, \ca))}_{(c)} ],\label{eq:final-obj}
\end{align} \normalsize
where
\begin{equation}
    w = \E_{p(\na \mid \ns)} \bigg[\frac{C_\theta^\pi(\ns, \na)}{1 - C_\theta^\pi(\ns, \na)} \bigg] \label{eq:w}
\end{equation}
is the classifier's prediction (ratio) at the next time step. Our resulting method can be viewed as a temporal difference~\citep{sutton1995td} approach to classifying future events. We will refer to our method as \textbf{recursive classification of examples (RCE)}. 
This equation has an intuitive interpretation. The first term \emph{(a)} trains the classifier to predict $1$ for the success examples themselves, and the third term \emph{(c)} trains the classifier to predict $0$ for random transitions. The important term is the second term \emph{(b)}, which is analogous to the ``bootstrapping'' term in temporal difference learning~\citep{sutton1988learning}. Term \emph{(b)} indicates that the probability of future success depends on the probability of success at the \emph{next} time step, as inferred using the classifier's own predictions.

Our resulting method is similar to existing actor-critic RL algorithms. 
To highlight the similarity to existing actor-critic methods, we can combine the \emph{(b)} and \emph{(c)} terms in the classifier objective function (Eq.~\ref{eq:final-obj}) to express the loss function in terms of two cross entropy losses:
\footnotesize \begin{align}
    \min_\theta \; & (1 - \gamma) \E_{\substack{p(\cs \mid \e =1),\\\ca \sim \pi(\ca \mid \cs)}} \left[ \CE(C_\theta^\pi(\cs, \ca); y = 1) \right] + (1 + \gamma w) \E_{p(\cs, \ca, \ns)} \Big[ \CE \big( C_\theta^\pi(\cs, \ca); y = \frac{\gamma w}{\gamma w + 1} \big) \Big]. \label{eq:ce}
\end{align} \normalsize
These cross entropy losses update the classifier to predict $y = 1$ for the success examples and to predict $y = \frac{\gamma w}{1 + \gamma w}$ for other states.

\begin{figure}[t]
\vspace{-2em}
\begin{algorithm}[H]
\caption{Recursive Classification of Examples}\label{alg:method}
{\footnotesize
\begin{algorithmic}[]
\State \textbf{Input}: success examples $\gS^*$
\State Initialize policy $\pi_\phi(\ca \mid \cs)$, classifier $C_\theta^\pi(\cs, \ca)$, replay buffer $\gD$
\While{not converged}
\State Collect a new trajectory: $\gD \gets \gD \cup \{\tau \sim \pi_\phi\}$
\State Sample success examples: $\{\cs^{(1)} \sim \gS^*, \ca^{(1)} \sim \pi_\phi(\ca \mid  \cs^{(1)}) \}$
\State Sample transitions: $\{(\cs^{(2)}, \ca^{(2)}, \ns) \sim \gD, \na \sim \pi_\phi(\na \mid \ns)\}$
\State $w \gets \frac{C_\theta^\pi(\ns, \na)}{1 - C_\theta^\pi(\ns, \na)}$ \Comment{Eq.~\ref{eq:w}}
\State $\gL(\theta) \gets (1 - \gamma) \CE(C_\theta(\cs^{(1)}, \ca^{(1)}); y = 1) + (1 + \gamma w) \CE(C_\theta(\cs^{(2)}, \ca^{(2)}); y = \frac{\gamma w}{1 + \gamma w})$
\State Update classifier: $\theta \gets \theta + \eta \nabla_\theta \gL(\theta)$ \Comment{Eq.~\ref{eq:final-obj}}
\State Update policy: $\phi \gets \phi + \eta \nabla_\phi \E_{\pi_\phi(\ca \mid \cs)}[C_\theta(\cs, \ca)]$  %
\EndWhile
\State \textbf{return} $\pi_\phi$
\end{algorithmic}}
\end{algorithm}
\vspace{-2em}
\end{figure}

\paragraph{Algorithm summary.}
Alg.~\ref{alg:method} summarizes our method, which alternates between updating the classifier, updating the policy, and (optionally) collecting new experience.
We update the policy to choose actions that maximize the classifier's confidence that the task will be solved in the future: $\max_\phi \E_{\pi_\phi(\ca \mid \cs)}[C_\theta^\pi(\cs, \ca)]$.
Following prior work~\citep{williams1991function, fox2015taming}), we regularized the policy updates by adding an entropy term with coefficient $\alpha = 10^{-4}$.
We also found that using N-step returns significantly improved the results of RCE (see Appendix~\ref{appendix:ablation} for details and ablation experiments.).
Implementing our method on top of existing methods such as SAC~\citep{haarnoja2018soft} or TD3~\citep{fujimoto2018addressing} requires only changing the standard Bellman loss with the loss in Eq.~\ref{eq:ce}.
See Appendix~\ref{appendix:details} for implementation details; code is available on the project website.

\vspace{-0.2em}
\section{Analysis}
\vspace{-0.2em}

In this section, we prove that RCE satisfies many of the same convergence and optimality guarantees (for example-based control) that standard RL algorithms satisfy (for reward-based MDPs). These results are important as they demonstrate that formulating control in terms of data, rather than rewards, does preclude algorithms from enjoying strong theoretical guarantees.
Proofs of all results are given in Appendix~\ref{appendix:proofs}, and we include a further discussion of how RCE relates to prior work in Appendix~\ref{appendix:connections}.

\subsection{Bellman Equations and Convergence Guarantees}
To prove that RCE converges to the optimal policy, we will first show that RCE satisfies a new Bellman equation:
\begin{lemma} \label{lemma:bellman}
The Bayes-optimal classifier $C^\pi$ for policy $\pi$ satisfies the following identity:
\begin{align}
    \frac{C^\pi(\cs, \ca)}{1 - C^\pi(\cs, \ca)} =& (1 - \gamma) p(\e = 1 \mid \cs) + \gamma \E_{\substack{p(\ns \mid \cs, \ca) \\ \pi(\na \mid \ns)}} \left[\frac{C^\pi(\ns, \na)}{1 - C^\pi(\ns, \na)} \right]. \label{eq:bellman}
\end{align}
\end{lemma}
The proof combines the definition of the Bayes-optimal classifier with the assumption from Eq.~\ref{eq:assumption}. This Bellman equation is analogous to the standard Bellman equation for Q-learning, where the reward function is replaced by $(1 - \gamma)p(\e = 1 \mid \cs)$ and the Q function is parametrized as $Q_\theta^\pi(\cs, \ca) = \frac{C_\theta^\pi(\cs, \ca)}{1 - C_\theta^\pi(\cs, \ca)}$. While we do not know how to compute this reward function, the update rule for RCE is equivalent to doing value iteration using that reward function and that parametrization of the Q-function:
\begin{lemma} \label{lemma:equivalence}
In the tabular setting, the \textbf{expected} updates for RCE are equivalent to doing value iteration with the reward function $r(\cs) = (1 - \gamma) p(\e = 1 \mid \cs)$ and a Q-function parametrized as $Q_\theta^\pi(\cs, \ca) = \frac{C_\theta^\pi(\cs, \ca)}{1 - C_\theta^\pi(\cs, \ca)}$.
\end{lemma}
This result tells us that RCE is equivalent to maximizing the reward function $(1 - \gamma) p(\e = 1 \mid \cs)$; however, RCE does not require knowing $p(\e = 1 \mid \cs)$, the probability that each state solves the task. Since value iteration converges in the tabular setting, an immediate consequence of Lemma~\ref{lemma:equivalence} is that tabular RCE also converges:
\begin{corollary} \label{lemma:convergence}
RCE converges in the tabular setting.
\end{corollary}
So far we have analyzed the training process for the classifier for a fixed policy. We conclude this section by showing that optimizing the policy w.r.t. the classifier improves the policy's performance.
\begin{lemma} \label{lemma:pi}
Let policy $\pi(\ca \mid \cs)$ and success examples $\gS^*$ be given, and let $C^\pi(\cs, \ca)$ denote the corresponding Bayes-optimal classifier. Define the improved policy as acting greedily w.r.t. $C^\pi$: $\pi'(\ca \mid \cs) = \mathbbm{1}(\mathbf{a} = \argmax_a C^\pi(\cs, \mathbf{a}))$.
Then the improved policy is at least as good as the old policy at solving the task: $p^{\pi'}(\fe = 1) \ge p^{\pi}(\fe = 1)$.
\end{lemma}

\subsection{Robust Example-based Control}
\label{sec:robust}

In this section, we derive a principled solution for the case where $p_U(\cs)$ is not known, which will correspond to modifying the objective function for example-based control. However, we will argue that, in some conditions, the method proposed in Sec.~\ref{sec:method} is \emph{already} robust to unknown $p_U(\cs)$, if that method is used with online data collection. The goal of this discussion is to provide a theoretical relationship between our method and a robust version of example-based control that makes fewer assumptions about $p_U(\cs)$. This discussion will also clarify how changing assumptions on the user's capabilities can change the optimal~policy.

When introducing example-based control in Sec.~\ref{sec:problem}, we emphasized that we \emph{must} make an assumption to make the example-based control problem well defined.
The exact probability that a success example solves the task depends on how often the user visited that state,
which the agent does not know. Therefore, there are many valid hypotheses for how likely each state is to solve the task. We can express the set of valid hypotheses using Bayes' Rule:
\begin{equation*}
\gP_{\e \mid \cs} \triangleq \bigg\{\hat{p}(\e = 1\mid \cs) = \frac{p_U(\cs \mid \e = 1) p(\e = 1)}{p_U(\cs)} \bigg \}.
\end{equation*}

Previously (Sec.~\ref{sec:method}), we resolved this ambiguity by assuming that $p_U(\cs)$ was equal to the distribution over states in our dataset of transitions.
As discussed in Sec.~\ref{sec:problem}, many problem settings violate this assumption, prompting us to consider the more stringent setting with no prior information about $p_U(\cs)$ (e.g., no prior knowledge about the user's capabilities).
To address this settting, we will assume the \emph{worst} possible choice of $p_U(\cs)$. This approach will make the agent robust to imperfect knowledge of the user's abilities and to mislabeled success examples.
Formally, we define the \emph{robust example-based control} problem as
\footnotesize \begin{align}
    \max_\pi \min_{\hat{p}(\e \mid \cs) \in \gP_{\e \mid \cs}} \E_{p^\pi(\fs)}[\hat{p}(\fe = 1 \mid \fs)] = \max_\pi \min_{p_U(\cs)} \E_{p^\pi(\fs)} \left[\frac{p_U(\cs \mid \e = 1)}{p_U(\cs)}p(\e = 1) \right]. \label{eq:robust-obj} 
\end{align} \normalsize
This objective can be understood as having the adversary assign a weight of $1 / p_U(\cs)$ to each success example.
The optimal adversary will assign lower weights to success examples that the policy frequently visits and higher weights to less-visited success examples.
Intuitively, the optimal policy should try to reach many of the success examples, not just the ones that are easiest to reach.
Thus, such a policy will continue to succeed even if certain success examples are removed, or are later discovered to have been mislabeled.
Surprisingly, solving this two-player game corresponds to minimizing an $f$-divergence:
\begin{lemma} \label{lemma:robust}
Define $H^2[p(\mathbf{x}), q(\mathbf{x})] = \int (\sqrt{p(\mathbf{x})} - \sqrt{q(\mathbf{x})})^2 d\mathbf{x}$ as the squared Hellinger distance, an $f$-divergence. Robust example-based control (Eq.~\ref{eq:robust-obj}) is equivalent to minimizing the squared Hellinger distance between policy's discounted state occupancy measure and the \textbf{conditional} distribution $p(\cs \mid \e = 1)$:
\begin{align*}
\min_{\hat{p}(\e \mid \cs) \in \gP_{\e \mid \cs}} \hspace{-0.8em} p^{\pi, \hat{p}}(\fe) = 1 \!-\! \frac{1}{2} H^2[p(\cs | \e = 1) , p^\pi(\fs = \cs)].
\end{align*}
\end{lemma}
\vspace{-1em}
The main idea of the proof (found in Appendix~\ref{appendix:robust}) is to compute the worst-case distribution $p_U(\cs)$ using the calculus of variations.
Preliminary experiments (Fig.~\ref{fig:robust} in Appendix~\ref{appendix:robust}) show that a version of RCE with online data collection finds policies that perform well on the robust example-based control objective (Eq.~\ref{eq:robust-obj}).
In fact, under somewhat stronger assumptions, we can show that the solution of robust example-based control is a fixed point of \emph{iterated} RCE (see Appendix~\ref{appendix:iterated}).
Therefore, in our experiments, we use RCE with online data~collection.

\vspace{-0.2em}
\section{Experiments}
\vspace{-0.2em}
\label{sec:experiments}

Our experiments study how effectively RCE solves example-based control tasks, especially in comparison to prior methods that learn an explicit reward function.
Both RCE and the prior methods receive only the success examples as supervision; no method has access to expert trajectories of reward functions. Additional experiments in Sec.~\ref{sec:experiments-image} study whether RCE can solve tasks using image observations. These experiments test whether RCE can solve tasks in new environments that are different from those where the success examples were collected, and test whether RCE learns policies that learn a general notion of success rather than just memorizing the success examples.
We include videos of learned policies online\footnote{
\iffinal
    \url{https://ben-eysenbach.github.io/rce}
\else
    \url{https://rce-anonymous.github.io/}
\fi
} and include implementation details, hyperparameters, ablation experiments, and a list of failed experiments in the Appendix.

We compare RCE against prior methods that infer a reward function from the success examples and then apply an off-the-shelf RL algorithm; some baselines iterate between these two steps.
AIRL~\citep{fu2017learning} is a popular adversarial imitation learning method. VICE~\citep{fu2018variational} is the same algorithm as AIRL, but intended to be applied to success examples rather than full demonstrations. We will label this method as ``VICE'' in figures, noting that it is the same algorithm as AIRL.
DAC~\citep{kostrikov2018discriminator} is a more recent, off-policy variant of AIRL. We also compared against two recent methods that learn rewards from \emph{demonstrations}:
ORIL~\citep{zolna2020offline} and PURL~\citep{xu2019positive}.
Following prior work~\citep{konyushkova2020semi}, we also compare against ``frozen'' variants of some baselines that first train the parameters of the reward function and then apply RL to that reward function without updating the parameters of the reward function again.
Our method differs from these baselines in that we do not learn a reward function from the success examples and then apply RL, but rather learn a policy directly from the success examples.
Lastly, we compare against SQIL~\citep{reddy2020sqil}, an imitation learning method that assigns a reward of $+1$ to states from demonstrations and $0$ to all other states. SQIL does not learn a separate reward function and structurally resembles our method, but is derived from different principles (see Sec.~\ref{sec:related-work}.).

\begin{figure*}[!t]
    \centering
    \vspace{-2.5em}
    \includegraphics[width=\linewidth]{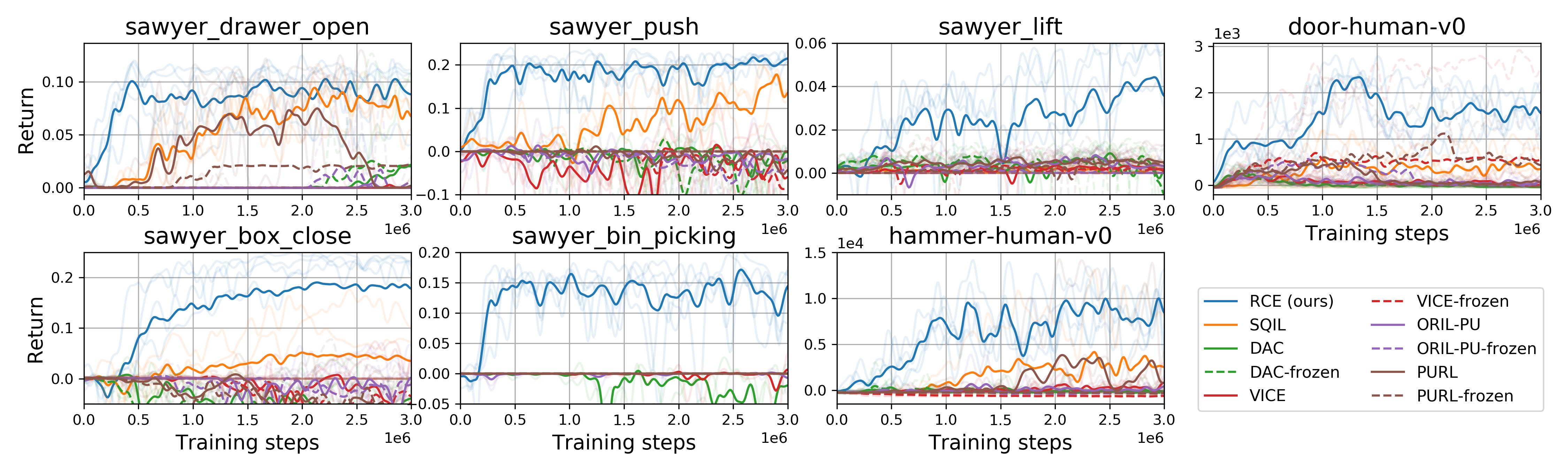}
    \vspace{-1.5em}
    \caption{{\footnotesize \textbf{Recursive Classification of Examples for learning manipulation tasks}: We apply RCE to a range of manipulation tasks, each accompanied with a dataset of success examples. For example, on the \texttt{sawyer\_lift} task, we provide success examples where the object has been lifted above the table. We use the cumulative task return ($\uparrow$ is better) solely for evaluation. Our method (blue line) outperforms prior methods across all tasks.}}
    \vspace{-1.5em}
    \label{fig:manipulation}
\end{figure*}

\vspace{-0.1em}
\subsection{Evaluating RCE for Example-Based Control.}
\vspace{-0.1em}

\begin{wrapfigure}[11]{R}{0.5\textwidth}
    \centering
    \vspace{-0.8em}
    \includegraphics[width=0.24\linewidth]{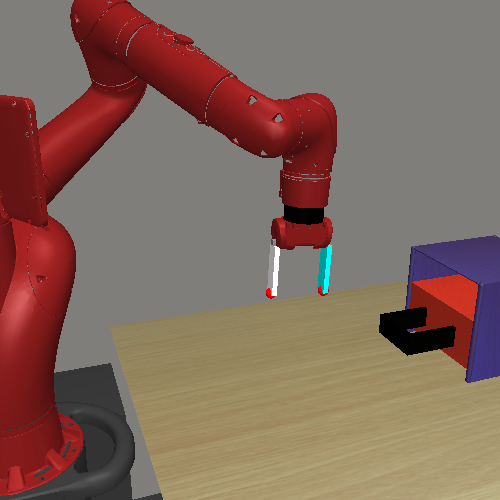}
    \includegraphics[width=0.24\linewidth]{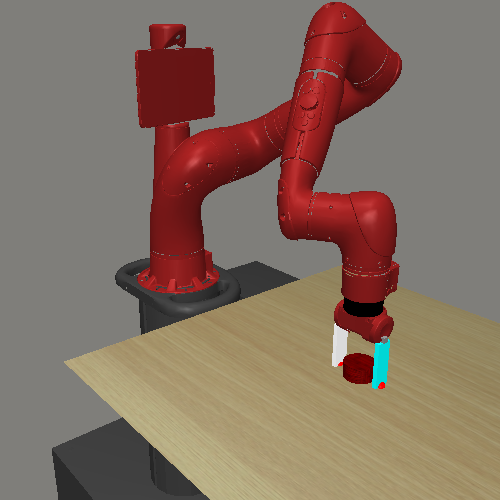}
    \includegraphics[width=0.24\linewidth]{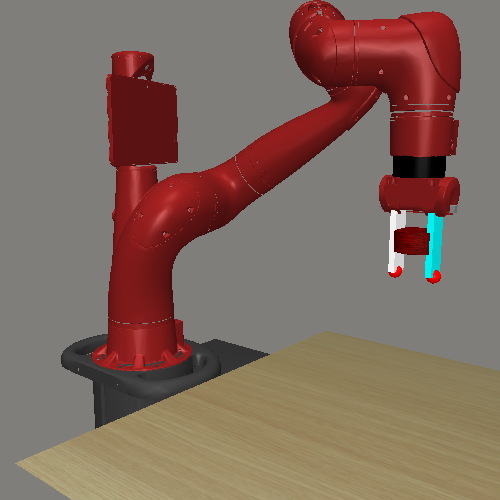}
    \includegraphics[width=0.24\linewidth]{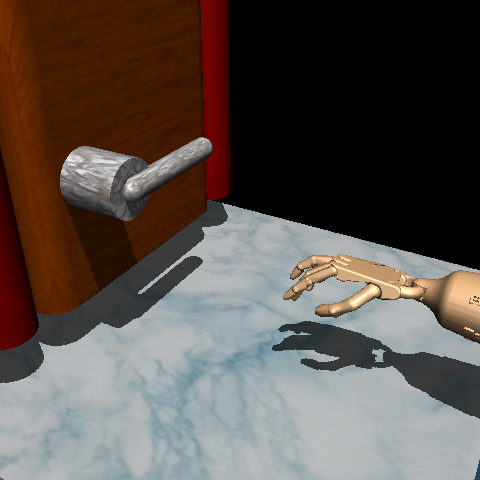}
    \includegraphics[width=0.24\linewidth]{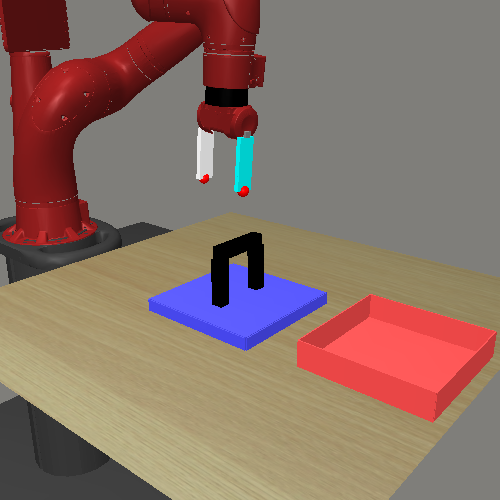}
    \includegraphics[width=0.24\linewidth]{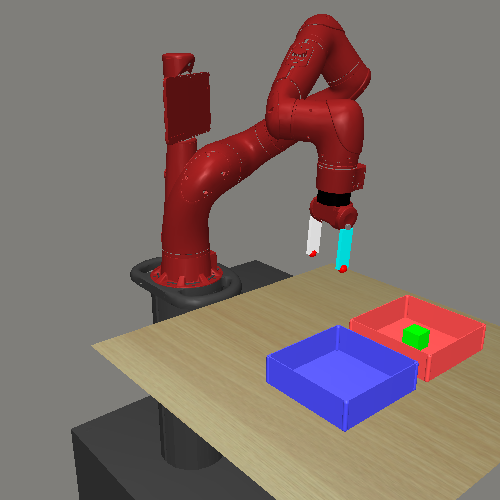}
    \includegraphics[width=0.24\linewidth]{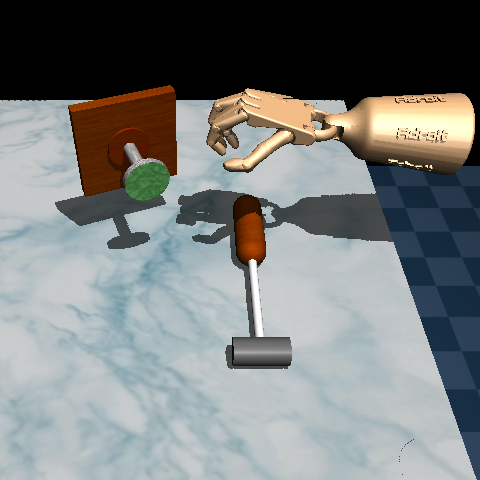}
    \vspace{-0.5em}
    \caption{{\footnotesize \textbf{Manipulation Environments}}}
    \label{fig:environments}
    \vspace{-1.em}
\end{wrapfigure}

We evaluate each method on five Sawyer manipulation tasks from Meta-World~\citep{yu2020meta} and two manipulation tasks from~\citet{rajeswaran18learning}. Fig.~\ref{fig:environments} illustrates these tasks. On each task, we provide the agent with 200 successful outcomes to define the task. For example, on the \texttt{open\_drawer} task, these success examples show an opened drawer.
As another example, on the \texttt{sawyer\_push} task, success examples not only have the end effector touching the puck, but (more importantly) the puck position is different.
We emphasize that these success examples only reflect the final state where the task is solved and are not full expert trajectories. This setting is important in practical use-cases: it is often easier for humans to arrange the workspace into a successful configuration than it is to collect an entire demonstration. See Appendix~\ref{appendix:environments} for details on how success examples were generated for each task. 
While these tasks come with existing user-defined reward functions, these rewards are not provided to any of the methods in our experiments and are used solely for evaluation ($\uparrow$ is better). We emphasize that this problem setting is exceedingly challenging: the agent provided only with examples of success states (e.g., an observation where an object has been placed in the correct location). Most prior methods that tackle similar tasks employ hand-designed reward functions or distance functions, full demonstrations, or carefully-constructed initial state distributions.

The results in Fig.~\ref{fig:manipulation} show that RCE significantly outperforms prior methods across all tasks.
The transparent lines indicate one random seed, and the darker lines are the average across random seeds. RCE solves many tasks, such as bin picking and hammering, that \emph{none} of the baselines make \emph{any} progress towards solving. The most competitive baseline, SQIL, only makes progress on the easiest two tasks; even on those tasks, SQIL learns more slowly than RCE and achieves lower asymptotic return.
To check that all baselines are implemented correctly, we confirm that all can solve a very simple reaching task described in the next section.

\subsection{Example-Based Control from Images}
\label{sec:experiments-image}

Our second set of experiments studies whether RCE can learn image-based tasks and assesses the generalization capabilities of our method. We designed three \emph{image-based} manipulation tasks. The \texttt{reach\_random\_position} task entails reaching a red puck, whose position is randomized in each episode. The \texttt{reach\_random\_size} task entails reaching a red object, but the actual shape of that object varies from one episode to the next. Since the agent cannot change the size of the object and the size is randomized from one episode to the next, it is impossible to reach any of the previously-observed success examples. To solve this task, the agent must learn a notion of success that is more general than reaching a fixed goal state. The third task, \texttt{sawyer\_clear\_image}, entails clearing an object off the table, and is mechanically more challenging than the reaching tasks.

\begin{figure*}[!t]
    \centering
    \vspace{-2em}
    \begin{subfigure}[c]{0.33\textwidth}
    \centering
    \caption*{{\scriptsize \texttt{sawyer\_reach\_random\_position\_image}}}
    \vspace{-0.5em}
    \fcolorbox{cyan}{cyan}{\includegraphics[width=0.2\linewidth]{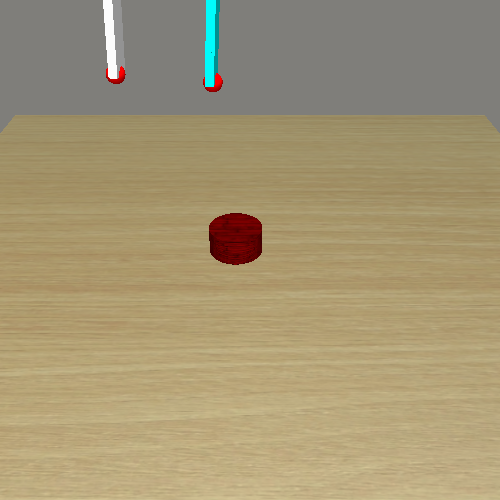}}
    \vrule \hspace{-0.05em}
    \fcolorbox{green}{green}{\includegraphics[width=0.2\linewidth]{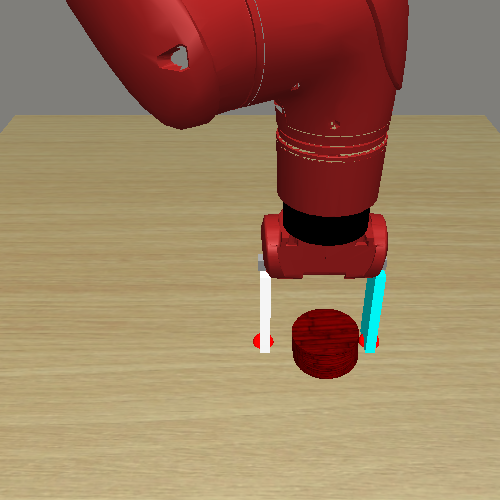}} \fcolorbox{green}{green}{\includegraphics[width=0.2\linewidth]{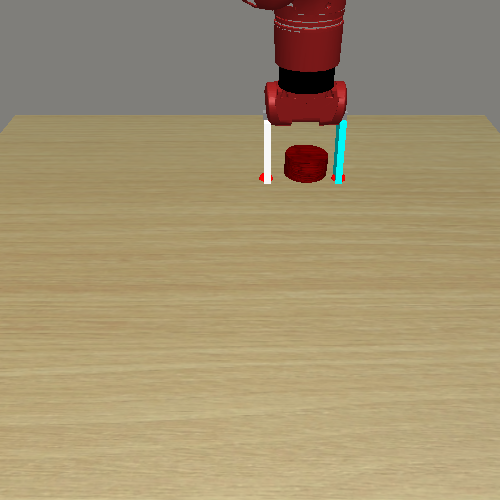}}
    \fcolorbox{green}{green}{\includegraphics[width=0.2\linewidth]{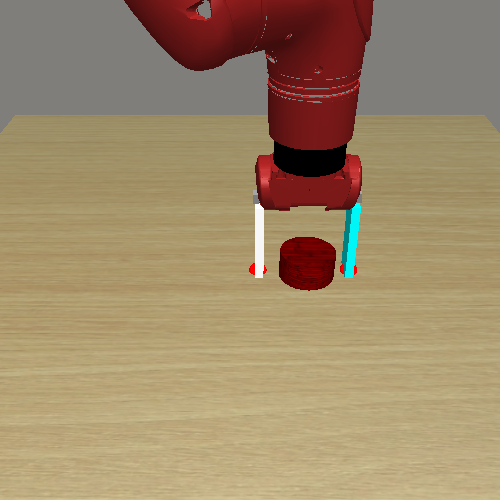}}
    \end{subfigure}%
    ~
    \begin{subfigure}[c]{0.33\textwidth}
    \centering
    \caption*{{\scriptsize \texttt{sawyer\_reach\_random\_size\_image}}}
    \vspace{-0.5em}
    \fcolorbox{cyan}{cyan}{\includegraphics[width=0.2\linewidth]{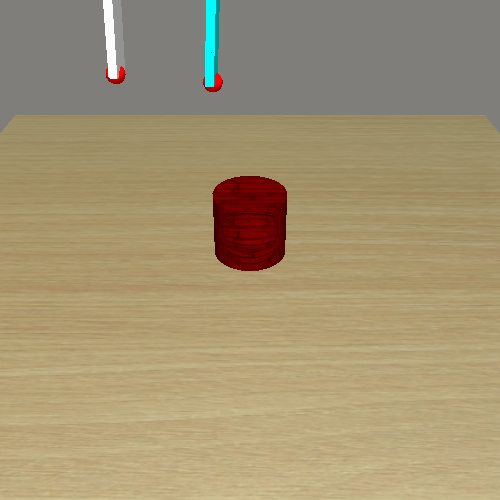}}
    \vrule \hspace{-0.05em}
    \fcolorbox{green}{green}{\includegraphics[width=0.2\linewidth]{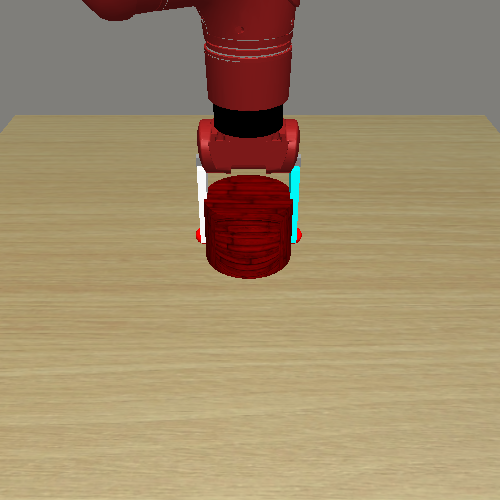}} \fcolorbox{green}{green}{\includegraphics[width=0.2\linewidth]{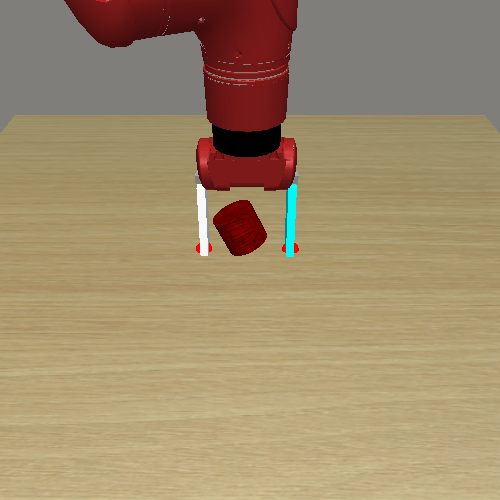}}
    \fcolorbox{green}{green}{\includegraphics[width=0.2\linewidth]{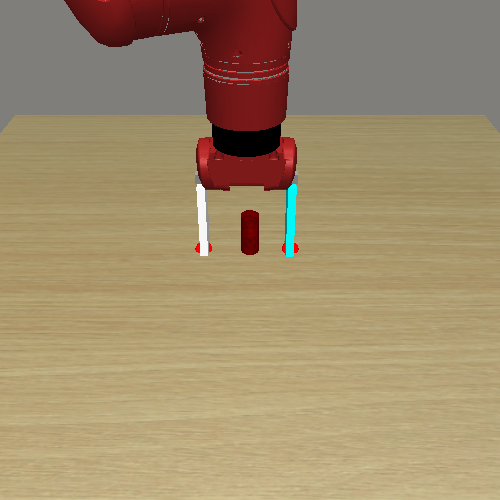}}
    \end{subfigure}%
    ~
    \begin{subfigure}[c]{0.33\textwidth}
    \centering
    \caption*{{\scriptsize \texttt{sawyer\_clear\_image}}}
    \vspace{-0.5em}
    \fcolorbox{cyan}{cyan}{\includegraphics[width=0.2\linewidth]{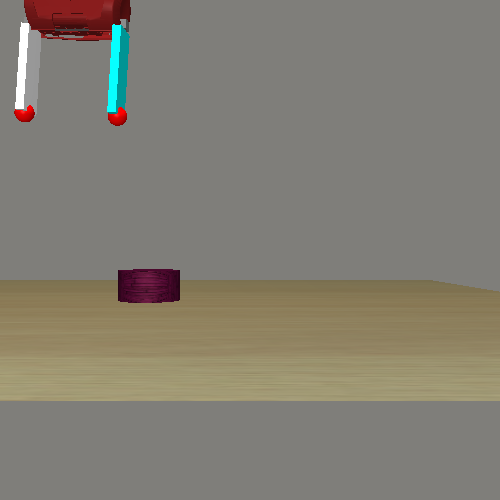}}
    \vrule \hspace{-0.05em}
     \fcolorbox{green}{green}{\includegraphics[width=0.2\linewidth]{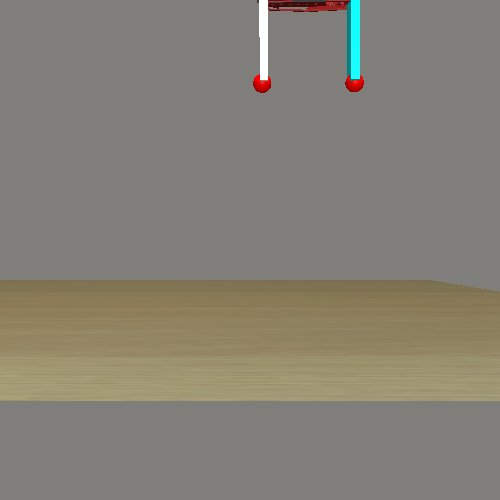}}
     \fcolorbox{green}{green}{\includegraphics[width=0.2\linewidth]{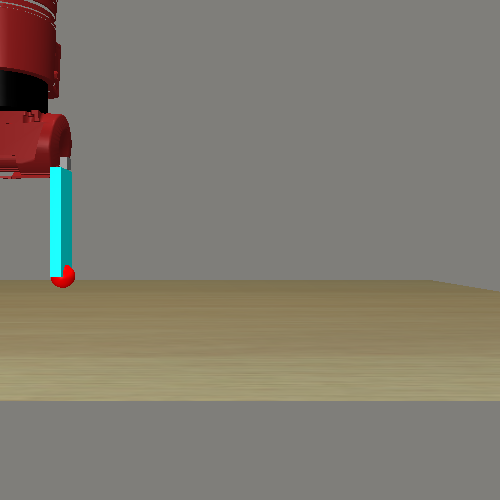}}
     \fcolorbox{green}{green}{\includegraphics[width=0.2\linewidth]{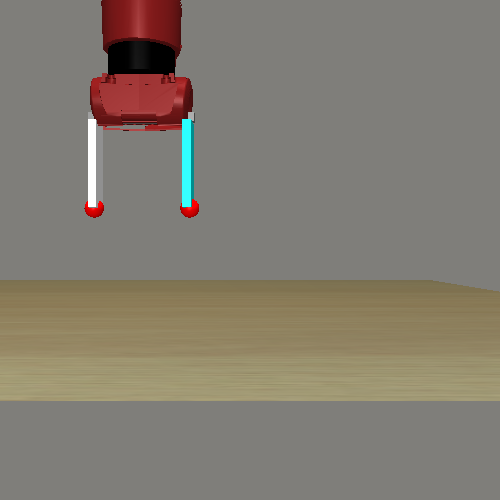}}
    \end{subfigure}
    
    \begin{subfigure}[c]{0.33\textwidth}
    \includegraphics[width=\linewidth]{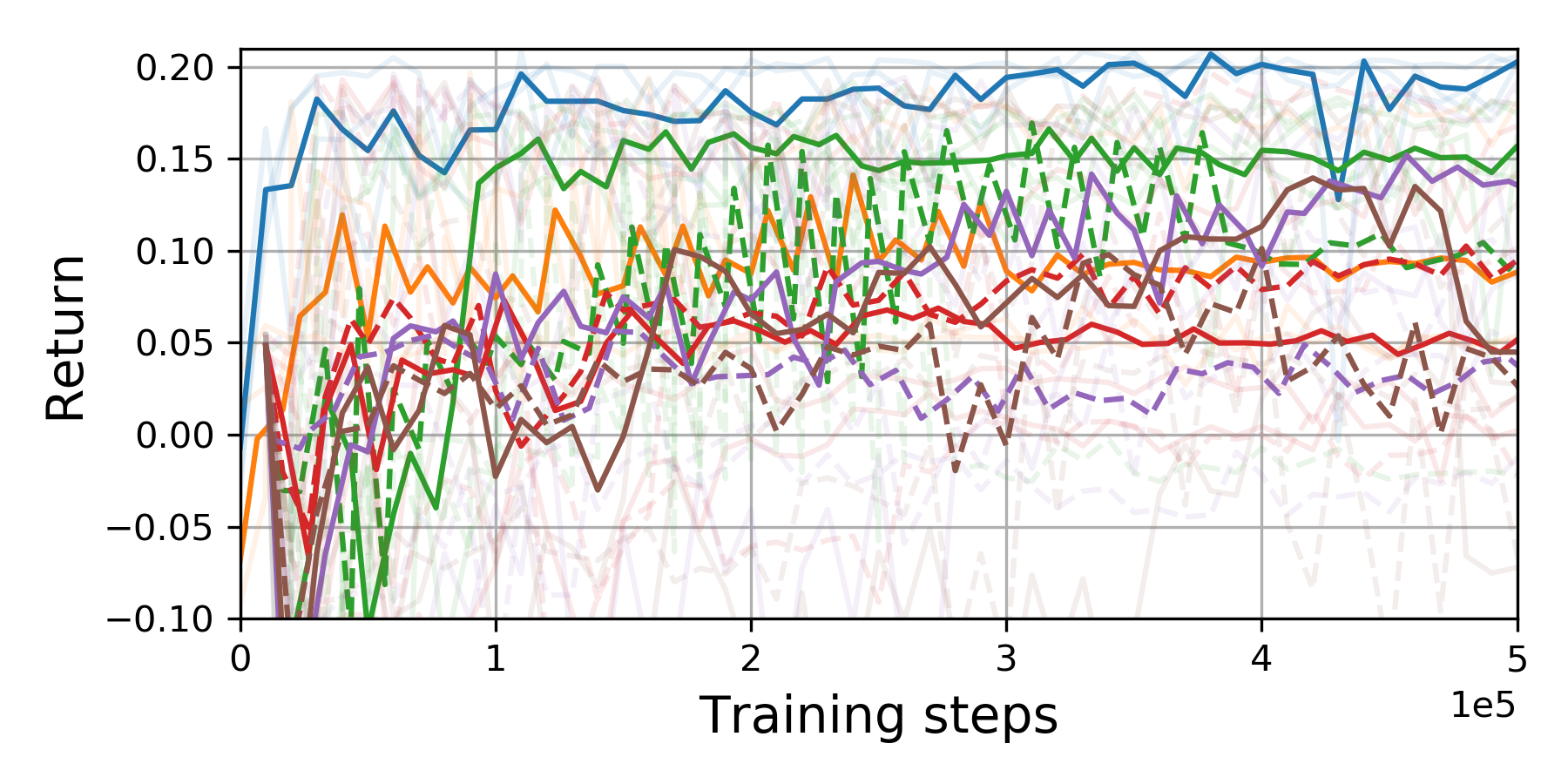}
    \end{subfigure}%
    ~
    \begin{subfigure}[c]{0.33\textwidth}
    \includegraphics[width=\linewidth]{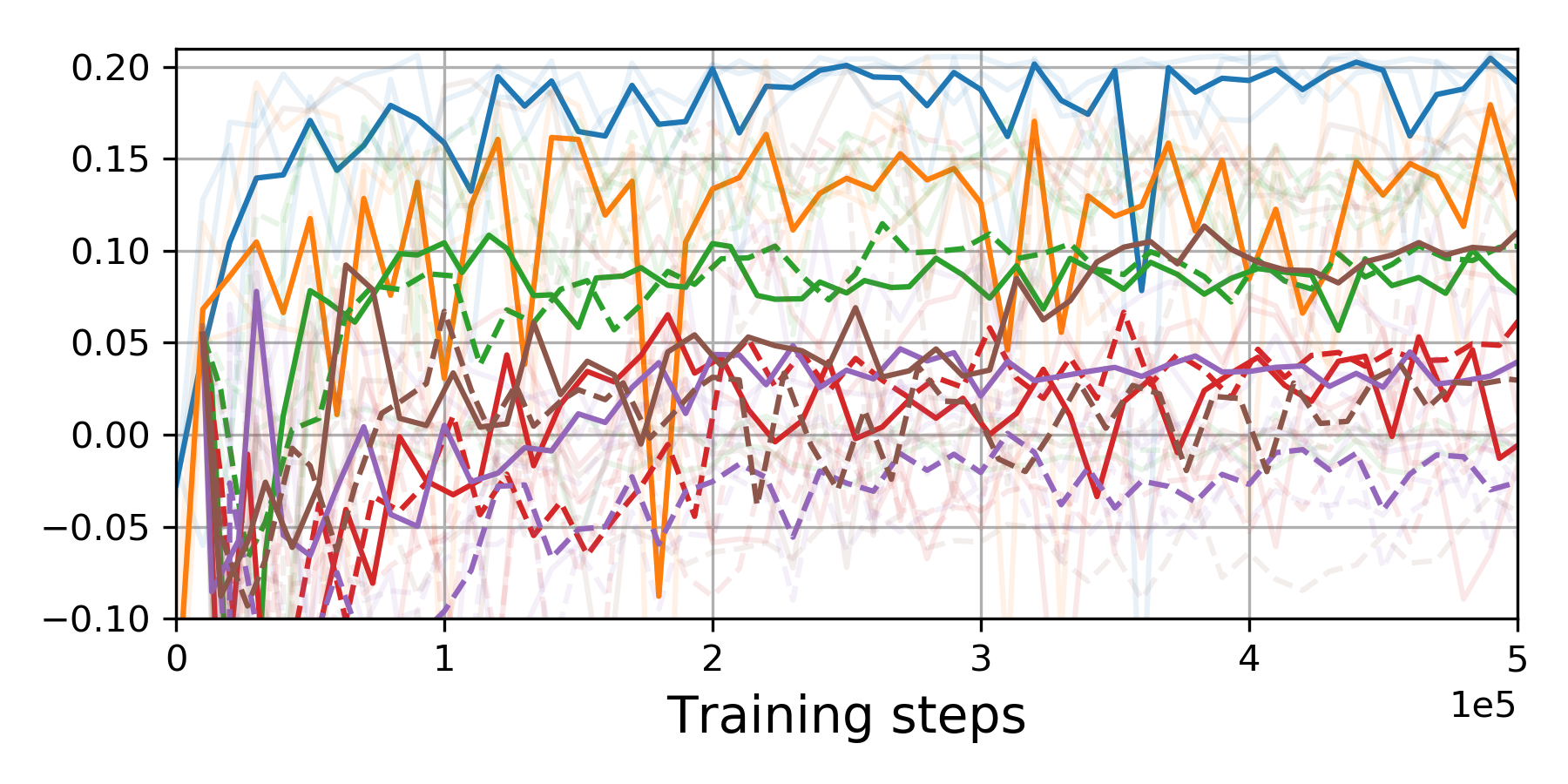}
    \end{subfigure}%
    ~
    \begin{subfigure}[c]{0.33\textwidth}
    \includegraphics[width=\linewidth]{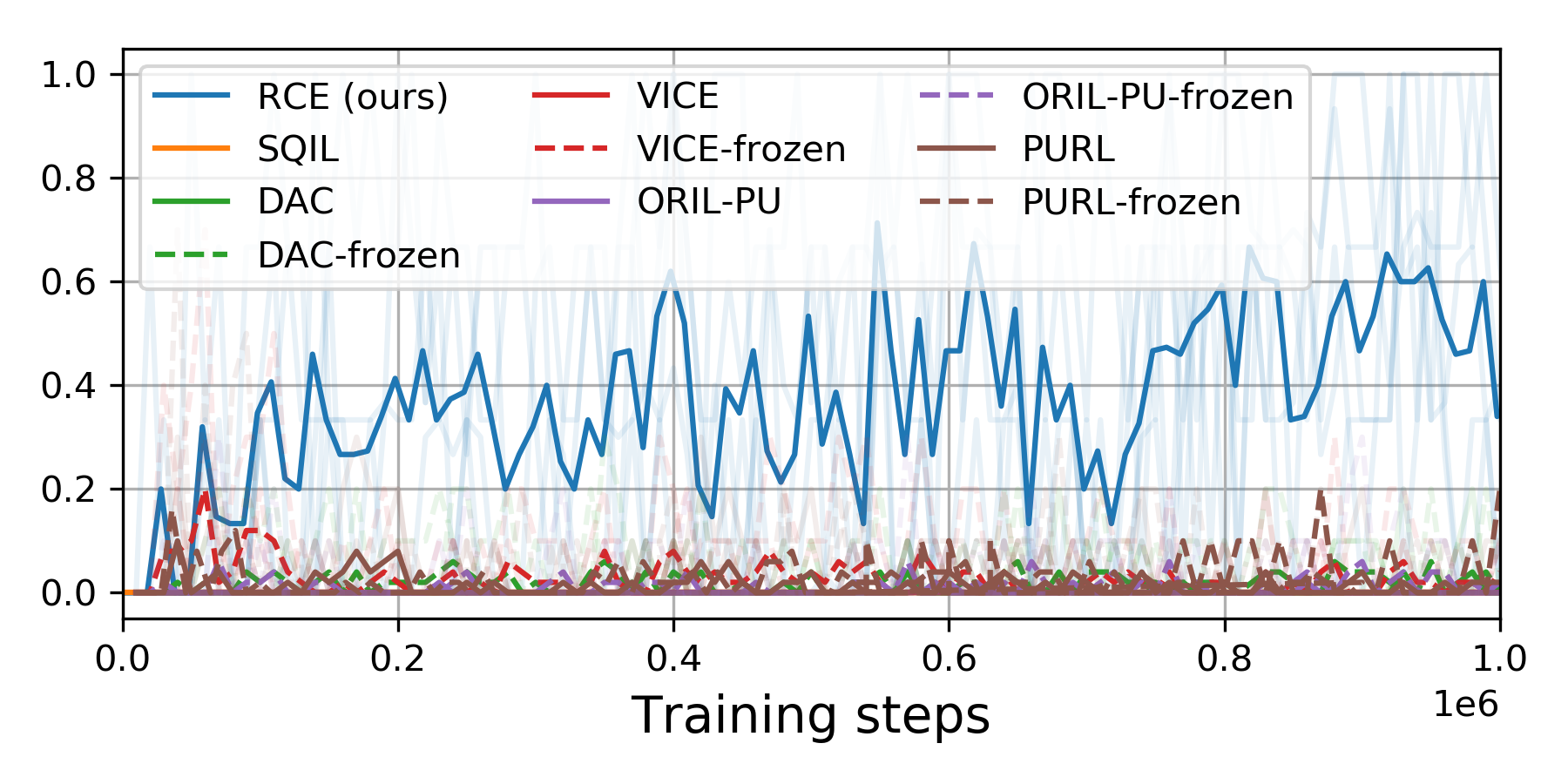}
    \end{subfigure}
    \caption{{\footnotesize \textbf{Example-based control from images}: We evaluate RCE on three manipulation tasks using image-observations. \figtop \, We show examples of the \fcolorbox{cyan}{white}{initial\phantom{p}state} and \fcolorbox{green}{white}{success examples} for each task. \figbottom \; RCE (blue line) outperforms prior methods, especially on the more challenging clearing task. For the \texttt{random\_size} task (\emph{center}), this entails reaching for new objects that have different sizes from any seen in the success examples.}}
    \vspace{-1.5em}
    \label{fig:reach-generalization}
\end{figure*}

Fig.~\ref{fig:reach-generalization} shows results from these image-based experiments, comparing RCE to the same baselines. We observe that RCE has learned to solve both reaching tasks, reaching for the object regardless of the location and size of the object. This task is mechanically easier than the state-based tasks in Fig.~\ref{fig:manipulation}, and all the baselines make some progress on this task, but learn more slowly than our method. The good performance of RCE on the \texttt{reach\_random\_size} task illustrates that RCE can solve tasks in a new environment, where the object size is different from that seen in the success examples. We hypothesize that RCE learns faster than these baselines because it ``cuts out the middleman,'' learning a value function directly from examples rather than indirectly via a separate reward function. To support this hypothesis, we note SQIL, which also avoids learning a reward function, learns faster than other baselines on these tasks. On the more challenging clearing task, only our method makes progress, suggesting that RCE is a more effective algorithm for learning these image-based control tasks. In summary, these results show that RCE outperforms prior methods at solving example-based control tasks from image observations, and highlights that RCE learns a policy that solves tasks in new environments that look different from any of the success examples.

\subsection{Ablation Experiments}
We ran seven additional experiments to study the importance of hyperparameters and design decisions. Appendix~\ref{appendix:ablation} provides full details and figures.
These experiments highlight that RCE is not an imitation learning method: RCE fails when applied to full expert trajectories, which are typically harder to provide than success examples.
Other ablation experiments underscore the importance of using n-step returns and validate the approximation made in Sec.~\ref{sec:method}.

\vspace{-0.5em}
\section{Conclusion}
\vspace{-0.5em}
\label{sec:conclusion}

In this paper, we proposed a data-driven approach to control, where examples of success states are used in place of a reward function. Our method estimates the probability of reaching a success example in the future and optimizes a policy to maximize this probability of success.
Unlike prior imitation learning methods, our approach is end-to-end and does not require learning a reward function. Our method is therefore simpler, with fewer hyperparameters and fewer lines of code to debug. Our analysis rests on a new data-driven Bellman equation, where example success states replace the typical reward function term. We use this Bellman equation to prove convergence of our classifier and policy.
We believe that formulating control problems in terms of data, rather than the reward-centric MDP, better captures the essence of many real-world control problems and suggests a new set of attractive learning algorithms.

\textbf{Limitations and future work.}
One limitation with RCE is that, despite producing effective policies, the classifier's predictions are not well calibrated. This issue resembles the miscalibration in Q-functions observed in prior work~\citep{lillicrap2015continuous, fujimoto2018addressing}.
Second, both RCE and the baselines we compared against all struggled to learn more challenging image-based tasks, such as image-based versions of the tasks shown in Fig.~\ref{fig:environments}. Techniques such as explicit representation learning~\citep{nachum2018near, srinivas2020curl} may be important for scaling example-based control algorithms to high-dimensional tasks.

\iffinal
    {\footnotesize
    \paragraph{Acknowledgements.}
    This work is supported by the Fannie and John Hertz Foundation, NSF (DGE1745016, IIS1763562), ONR (N000141812861), and US Army. We thank Ilya Kostrikov for providing the implementation of DAC. We thank Ksenia Konyushkova, 
    Konrad Zolna, and Justin Fu for discussions about baselines and prior work. We thank Oscar Ramirez for help setting up the image-based experiments, Ryan Julian and anonymous reviewers for feedback on drafts of the paper, and Alex Irpan for help releasing code.}
\fi

\clearpage
\appendix

\section{Derivation of Recursive Classification of Examples}
\label{appendix:derivation}
In this section we include the full derivation of our method (RCE).
Recall that we want to learn a classifier, $C_\theta^\pi(\cs, \ca)$ to discriminate between ``positive'' state-action pairs sampled from the conditional distribution $p^\pi(\cs, \ca \mid \fe = 1)$ and ``negatives'' sampled from a marginal distribution $p(\cs, \ca)$. We will use different class-specific weights, using a weight of $p(\fe = 1)$ for the ``positives'' and a weight of $1$ for the ``negatives.'' Then, if this classifier is trained perfectly, the Bayes-optimal solution is
\begin{equation*}
C_\theta^\pi(\cs, \ca) = \frac{p^\pi(\cs, \ca \mid \fe = 1) p(\fe = 1)}{p^\pi(\cs, \ca \mid\fe = 1) p(\fe = 1) + p(\cs, \ca)}.
\end{equation*}
The motivation for using these class-specific weights is so that the classifier's predicted probability ratio tells us the probability of solving the task in the future:
\begin{align*}
    \frac{C_\theta^\pi(\cs, \ca)}{1 - C_\theta^\pi(\cs, \ca)} &= p^\pi(\fe = 1 \mid \cs, \ca).
\end{align*}
Importantly, the resulting method will not actually require estimating the weight $p(\fe = 1)$.
We would like to optimize these parameters using maximum likelihood estimation:
\begin{align*}
    \gL^\pi (\theta) \triangleq \; & p(\fe = 1) \; \E_{p(\cs, \ca \mid \fe = 1)} [\log C_\theta^\pi(\cs, \ca)] + \E_{p(\cs, \ca)} [\log (1 - C_\theta^\pi(\cs, \ca))].
\end{align*}
While we can estimate the second expectation using Monte Carlo samples $(\cs, \ca) \sim p(\cs, \ca)$, we cannot directly estimate the first expectation because we cannot sample from \mbox{$(\cs, \ca) \sim p(\cs, \ca \mid \fe = 1)$}. We will address this challenge by deriving a method for training the classifier that resembles the temporal difference update used in value-based RL algorithms, such as Q-learning. We will derive our method from Eq.~\ref{eq:obj-1} using three steps.

The \textbf{first step} is is to factor the first expectation in Eq.~\ref{eq:obj-1}
\begin{equation*}
\text{\footnotesize $p(\cs, \ca \mid \fe = 1) p(\fe = 1) = p^\pi(\fe = 1 \mid \cs, \ca) p(\cs, \ca).$ }
\end{equation*}
Substituting this identity into Eq.~\ref{eq:obj-1}, we obtain
\begin{align}
    \gL^\pi (\theta)
    &= \E_{{\color{orange}p(\cs, \ca)}} [{\color{orange}p^\pi(\fe = 1 \mid \cs, \ca)} \log C_\theta^{(t)}] + \E_{\substack{p(\cs, \ca)}} [\log (1 - C_\theta^{(t)})] \nonumber \\
    &= \E_{p(\cs, \ca)} [p^\pi(\fe = 1 \mid \cs, \ca)\log C_\theta^{(t)} +  \log (1 - C_\theta^{(t)})]. \label{eq:step-1}
\end{align}

The \textbf{second step} is to observe that the term $p^\pi(\fe = 1 \mid \cs, \ca)$ satisfies the following recursive identity:
\begin{align}
    p^\pi(\fe = 1 \mid \cs, \ca) &= (1 - \gamma) p(\e = 1 \mid \cs) + \gamma \E_{\substack{p(\ns \mid \cs, \ca),\\ \pi_\phi(\na \mid \ns)}} \left[p^\pi(\fe = 1 \mid \ns, \na) \right]. \label{eq:td}
\end{align}
We substitute this identity into Eq.~\ref{eq:step-1} and then break the expectation into two terms.
{\footnotesize
\begin{align}
    \gL^\pi (\theta) &= \E_{p(\cs, \ca)}\left[ {\color{orange} \left((1 - \gamma) p(\e = 1 \mid \cs) + \gamma \E_{\substack{p(\ns \mid \cs, \ca),\\ \pi_\phi(\na \mid \ns)}} \left[p^\pi(\fe = 1 \mid \ns, \na) \right] \right)} \log C_\theta^{(t)} + \log (1 - C_\theta^{(t)}) \right] \nonumber \\
    &= (1 - \gamma) \E_{p(\cs, \ca)}\left[ p(\e = 1 \mid \cs) \log C_\theta^{(t)} \right] \nonumber \\
    & \qquad\qquad + \gamma \E_{p(\cs, \ca)} \left[ \E_{\substack{p(\ns \mid \cs, \ca),\\ \pi_\phi(\na \mid \ns)}} \left[p^\pi(\fe = 1 \mid \ns, \na) \right] \log C_\theta^{(t)} + \log (1 - C_\theta^{(t)}) \right]. \label{eq:obj-3}
\end{align}}

The \textbf{third step} is to express the first expectation in Eq.~\ref{eq:obj-3} in terms of samples of success example, $\mathbf{s^*} \sim p(\cs \mid \e = 1)$:
\begin{align*}
     \E_{p(\cs, \ca)}\left[ p(\e = 1 \mid \cs) \log C_\theta^{(t)} \right]
     &= \E_{p(\cs, \ca)}\left[ {\color{orange}\frac{p_U(\cs \mid \e = 1)}{p(\cs)}p_U(\e = 1)} \log C_\theta^{(t)} \right] \\
     &= p_U(\e = 1) \E_{\substack{p(\cs),\\p(\ca \mid \cs)}} \left[ \frac{p_U(\cs \mid \e = 1)}{p(\cs)} \log C_\theta^{(t)} \right] \\
     &= p_U(\e = 1) \E_{\substack{p_U(\cs \mid \e = 1),\\p(\ca \mid \cs)}} \left[ \log C_\theta^{(t)} \right].
\end{align*}
Note that actions are sampled from $p(\ca \mid \cs) = p(\cs, \ca) / p(\cs)$, the average policy used to collect the dataset of transitions. Finally, we substitute this expression back into Eq.~\ref{eq:obj-3}, and use Eq.~\ref{eq:posterior} to estimate the probability of future success at time $t+1$.
\begin{align*}
    \gL^\pi (\theta) &= (1 - \gamma) {\color{orange}p_U(\e = 1) \E_{\substack{p_U(\cs \mid \e = 1),\\p(\ca \mid \cs)}}} \left[ \log C_\theta^{(t)} \right] \\
    & \qquad\qquad + \gamma \E_{p(\cs, \ca)} \left[ \E_{\substack{p(\ns \mid \cs, \ca),\\ \pi_\phi(\na \mid \ns)}} \left[{\color{orange}\frac{C_\theta^{(t+1)}}{1 - C_\theta^{(t+1)}}} \right] \log C_\theta^{(t)} + \log (1 - C_\theta^{(t)}) \right].
\end{align*}
Thus, we have arrived at our final objective (Eq.~\ref{eq:final-obj}).

\section{Proofs}
\label{appendix:proofs}
In this section we include proofs of the theoretical results in the main text.

\paragraph{Proof of Lemma~\ref{lemma:bellman}}
\begin{proof}
To start, we recall that Eq.~\ref{eq:posterior} says that the Bayes-optimal classifier satisfies 
\begin{equation*}
    \frac{C^\pi(\cs, \ca)}{1 - C^\pi(\cs, \ca)} = p^\pi(\fe = 1 \mid \cs, \ca).
\end{equation*}
Substituting this identity into both the LHS and RHS of the recursive definition of $p^\pi(\fe = 1 \mid \cs, \ca)$ (Eq.~\ref{eq:td}), we obtain the desired result:
\begin{equation*}
    \frac{C^\pi(\cs, \ca)}{1 - C^\pi(\cs, \ca)} = (1 - \gamma) p^\pi(\e = 1 \mid \cs) + \gamma \E_{\substack{p(\ns \mid \cs, \ca) \\ \pi(\na \mid \ns)}} \left[\frac{C^\pi(\ns, \na)}{1 - C^\pi(\ns, \na)} \right].
\end{equation*}
\end{proof}

\paragraph{Proof of Lemma~\ref{lemma:convergence}}
\begin{proof}
We will show that RCE is equivalent to using  value iteration using the Bellman equation above, where $(1 - \gamma) p(\nee = 1 \mid \cs, \ca)$ takes the role of the reward function.  For given transition $(\cs, \ca, \ns)$, the corresponding TD target $y$ is the expected value of three terms (Eq.~\ref{eq:final-obj}): with weight $(1 - \gamma) p(\e = 1 \mid \cs)$ it is assigned $y = 1$; with weight $\gamma \E[w]$ it is assigned label $y = 1$; and with weight 1 it is assigned label $y = 0$. Thus, the expected value of the TD target $y$ can be written as follows:
\begin{equation*}
    \E[y \mid \cs, \ca, \ns] = \frac{(1 - \gamma) p(\e = 1 \mid \cs) \cdot 1 + \gamma \E[w] \cdot 1 + 1 \cdot 0}{(1 - \gamma) p(\e = 1 \mid \cs) + \gamma \E[w] + 1}.
\end{equation*}
Thus, the assignment equation can be written as follows:
\begin{equation}
    C^\pi(\cs, \ca) \gets \E[y \mid \cs, \ca, \ns] = \frac{(1 - \gamma) p(\e = 1 \mid \cs) + \gamma \E[w]}{(1 - \gamma) p(\e = 1 \mid \cs) + \gamma \E[w] + 1}.
\end{equation}
While this gives us an assignment equation for $C$, our Bellman equation is expressed in terms of the \emph{ratio} $C / (1 - C)$. Noting that the function $C / (1 - C)$ is strictly monotone increasing, we can write the same assignment equation for the ratio as follows:
\begin{align*}
   \frac{C^\pi(\cs, \ca)}{1 - C^\pi(\cs, \ca)} \gets \frac{\E[y \mid \cs, \ca, \ns]}{1 - \E[y \mid \cs, \ca, \ns]} = (1 - \gamma) p(\e = 1 \mid \cs) + \gamma \E[w].
\end{align*}
Recalling that $w = \frac{C(\ns, \na)}{1 - C(\ns, \na)}$, we conclude that our temporal difference method is equivalent to doing value iteration using the Bellman equation in Eq.~\ref{eq:bellman}. 
\end{proof}

\paragraph{Proof of Corollary~\ref{lemma:convergence}}
\begin{proof}
Lemma~\ref{lemma:equivalence} showed that the update rule for our method is equivalent to value iteration. Value iteration is known to converge in the tabular setting~\citep[Theorem 1]{jaakkola1994convergence}, so our method also converges.
\end{proof}

\paragraph{Proof of Lemma~\ref{lemma:pi}}

The proof of Lemma~\ref{lemma:pi} is nearly identical to the standard policy improvement proof for Q-learning.
\begin{proof}
\begin{align*}
    p^\pi(\fe = 1) 
    &= \E_{p_0(s), \pi(a_0 \mid s_0)}[p^\pi(\fe = 1 \mid s_0, a_0)] \\
    &\le \E_{p_0(s), \pi'(a_0 \mid s_0)}[p^\pi(\fe = 1 \mid s_0, a_0)] \\
    &= \E_{p_0(s), \pi'(a_0 \mid s_0)}[(1 - \gamma) p(e_0 = 1 \mid s_0, a_0) + \gamma \E_{p(s_1 \mid s_0, a_0), \pi(a_1 \mid s_1)} [p^\pi(\fe = 1 \mid s_1, a_1)]] \\
    &\le \E_{p_0(s), \pi'(a_0 \mid s_0)}[(1 - \gamma) p(e_0 = 1 \mid s_0, a_0) + \gamma \E_{p(s_1 \mid s_0, a_0), \pi'(a_1 \mid s_1)} [p^\pi(\fe = 1 \mid s_1, a_1)]] \\
    &= \E_{\substack{p_0(s), \pi'(a_0 \mid s_0) \\ p(s_1 \mid s_0, a_0), \pi'(a_1 \mid s_1)}}[(1 - \gamma) p(e_0 = 1 \mid s_0, a_0) + \gamma p(e_1 = 1 \mid s_1, a_1) \\
    & \hspace{10em} + \gamma^2 \E_{p(s_2 \mid s_1, a_1), \pi(a_2 \mid s_2)} [p^\pi(\fe = 1 \mid s_2, a_2)]] \\
    & \cdots \\
    & \le p^{\pi'}(\fe = 1)
\end{align*}
\end{proof}

\section{Robust Example-Based Control}
\label{appendix:robust}

Our strategy for analyzing robust example-based control will be to first compute an analytic solution to the inner minimization problem. Plugging in the optimal adversary, we will find that we weight the success examples inversely based on how often our policy utilizes each of the potential solution strategies.

\subsection{Proof of Lemma~\ref{lemma:robust}}
\begin{proof}

First, we solve the inner minimization problem. Note that the adversary's choice of an example probability function $\hat{p}(\e \mid \cs)$ is equivalent to a choice of a marginal distribution, $p_U(\cs)$. We can this write the inner minimization problem as follows:
\begin{align*}
    \min_{p_U(\cs)} \int p^\pi(\fs = \cs) \frac{p_U(\cs \mid \e = 1) \cancel{p(\e = 1)}}{p_U(\cs)} d\cs
\end{align*}
As before, we can ignore the constant $p(\e = 1)$. We solve this optimization problem, subject to the constraint that $\int p_U(\cs) d\cs = t$, using calculus of variations. The corresponding Lagrangian is
\begin{equation*}
    \gL(p_U, \lambda) = \int \frac{p^\pi(\fs = \cs) p_U(\cs \mid \e = 1)}{p_U(\cs)} d \cs + \lambda \left (\int p_U(\cs) d\cs - 1 \right).
\end{equation*}
Setting the derivative $\frac{d \gL}{d p_U(\cs)} = 0$, we get
\begin{equation*}
    -\frac{p^\pi(\fs = \cs) p_U(\cs \mid \e = 1)}{p_U^2(\cs)} + \lambda  = 1 \implies p_U(\cs) = \frac{\sqrt{p^\pi(\fs = \cs) p_U(\cs \mid \e = 1)}}{\int \sqrt{p^\pi(\fs = \cs') p_U(\cs' \mid \e = 1)} d \cs'}.
\end{equation*}
Note that $\frac{d^2 \gL}{d p_U(\cs)^2} > 0$, so this stationary point is a minimum.
We can therefore write the worst-case probability function as 
\begin{align*}
    \hat{p}(\e = 1 \mid \cs) &= \frac{p_U(\cs \mid \e = 1)}{p_U(\cs)} p(\e = 1) \\
    &= \sqrt{\frac{p_U(\cs \mid \e = 1)}{p^\pi(\fs = \cs)}} \int \sqrt{p^\pi(\fs = \cs') p_U(\cs' \mid \e = 1)} d \cs' p(\e = 1).
\end{align*}
Intuitively, this says that the worst-case probability function is one where successful states $\cs \sim p_U(\cs \mid \e = 1)$ are downweighted if the current policy visits those states more frequently (i.e., if $p^\pi(\fs = \cs)$ is large).
Substituting this worst-case example probability into Eq.~\ref{eq:robust-obj}, we can write the adversarial objective as follows:
\begin{align*}
    \max_\pi & \int p^\pi(\fs = \cs) \sqrt{\frac{p_U(\cs \mid \e = 1)}{p^\pi(\fs = \cs)}} d 
    \cs \int \sqrt{p^\pi(\fs = \cs') p_U(\cs' \mid \e = 1)} d \cs' p(\e = 1) \\
    &= \left( \int \sqrt{p^\pi(\fs = \cs) p_U(\cs \mid \e = 1)} d \cs\right)^2 p(\e = 1)
\end{align*}
Since $p(\e = 1)$ is assumed to be a constant and $(\cdot)^2$ is a monotone increasing function (for non-negative arguments), we can express this same optimization problem as follows:
\begin{align*}
    \max_\pi \int \sqrt{p^\pi(\fs = \cs) p_U(\cs \mid \e = 1)} d \cs.
\end{align*}
Using a bit of algebra, we can show that this objective is a (scaled and shifted) squared Hellinger distance:
\begin{align*}
    \int & \sqrt{p^\pi(\fs = \cs) p_U(\cs \mid \e = 1)} d \cs \\
    &= 1 - \frac{1}{2}\int \left( \sqrt{p^\pi(\fs = \cs)} - \sqrt{p_U(\cs \mid \e = 1)} \right)^2 d \cs \\
    &= 1 - \frac{1}{2}\int \left(\sqrt{p_U(\cs \mid \e = 1)} \left( \sqrt{\frac{p^\pi(\fs = \cs)}{p_U(\cs \mid \e = 1)}} - 1\right)  \right)^2 d \cs \\
    &= 1 - \frac{1}{2} \E_{p_U(\cs \mid \e = 1)} \left[\left( 1 - \sqrt{\frac{p^\pi(\fs = \cs)}{p_U(\cs \mid \e = 1)}} \right)^2 \right] \\
    &= 1 - \frac{1}{2} H^2(p_U(\cs \mid \e = 1) , p^\pi(\fs = \cs)).
\end{align*}
\end{proof}

\subsection{Robust Example-Based Control and Iterated RCE}
\label{appendix:iterated}

\begin{wrapfigure}{R}{0.5\textwidth}
\vspace{-1.5em}
\begin{subfigure}[c]{0.5\linewidth}
\includegraphics[width=\linewidth]{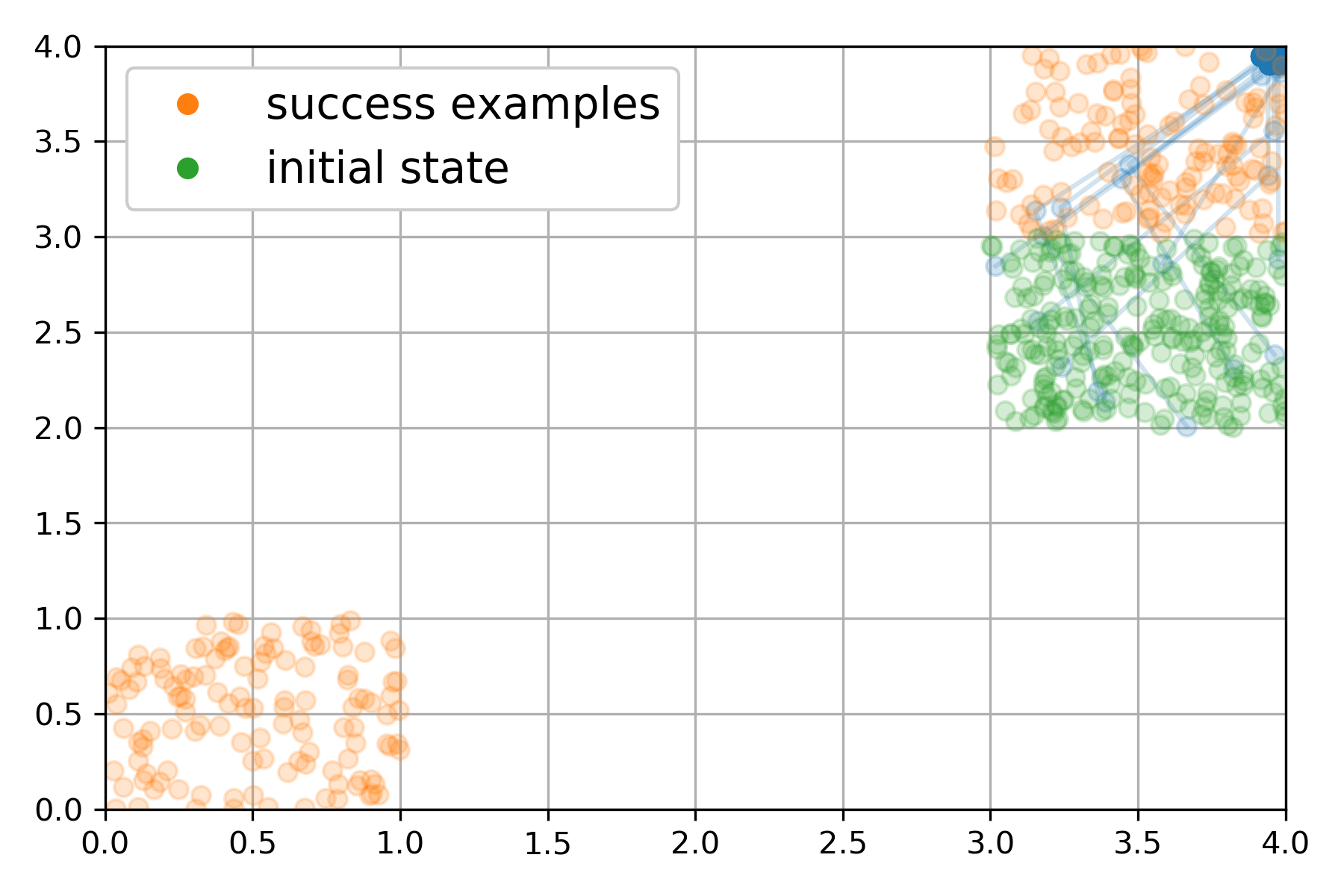}
\caption{Offline RCE}
\end{subfigure}%
~
\begin{subfigure}[c]{0.5\linewidth}
\includegraphics[width=\linewidth]{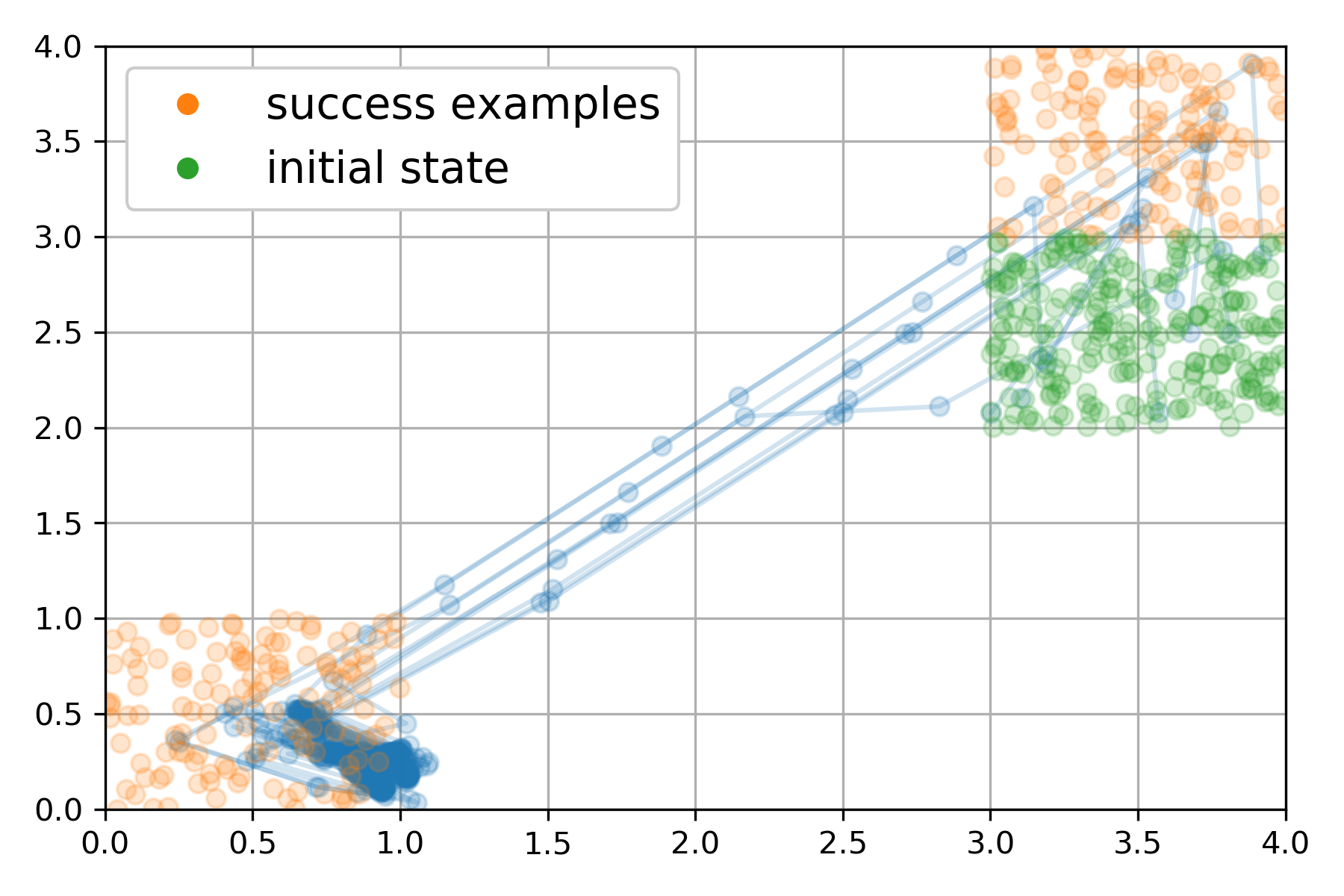}
\caption{Iterated RCE}
\end{subfigure}
\caption{{\footnotesize A 2D navigation task with two sets of success states. \figleft \, When we apply RCE to this task without online data collection, the agent navigates directly toward the closer set of success states. \figright \, When we run RCE in an iterated fashion, periodically collecting new transitions using the current policy, the learned policy visits both sets of success states. The behavior of iterated RCE is similar to what we would expect from robust example-based control.}}
\label{fig:robust}
\end{wrapfigure}
As noted in Sec.~\ref{sec:robust}, we found empirically that the policy found by RCE with online data collection was similar to the optimal policy for the robust example-based control objective. We visualize that experiment in Fig.~\ref{fig:robust}. In this 2D navigation task there are two sets of success examples (orange circles). When we apply RCE to this task using a uniform distribution of transitions, the agent navigates directly towards the closer success examples. However, when we run RCE in an iterative fashion, periodically collecting data from the current policy, the final policy visits both sets of success examples. This behavior is similar to what we would expect from the optimal policy to robust example-based control
We hypothesize that this happens because our implementation of RCE periodically collects new data using the current policy, violating the assumption that the dataset of transitions $p(\cs, \ca, \ns)$ is fixed. When the policy frequently visits one success example, that state will be included many times in the dataset of transitions $p(\cs, \ca, \ns)$ at the next iteration, so the classifier's predictions will decrease for that state.

We now show that (under a very strong assumption), the solution to the robust example-based control problem is a fixed point of an iterative version of RCE. The iterated version of RCE is defined by alternating between collecting a dataset of transitions using the current policy and then running RCE (Alg.~\ref{alg:method}) on those trajectories:
\begin{equation}
  \pi \tikzmark{a} \gets \text{RCE}(\gS^*, \gD) , \qquad \gD \gets \tikzmark{b} \{\tau \sim \pi\}.
  \begin{tikzpicture}[overlay,remember picture,out=315,in=225,distance=1cm]
    \draw[<-,blue,shorten >=5pt,shorten <=5pt, line width=0.5mm] (a.center) to (b.center);
  \end{tikzpicture}
  \begin{tikzpicture}[overlay,remember picture,out=45,in=135,distance=1cm]
    \draw[->,blue,shorten >=5pt,shorten <=5pt, line width=0.5mm] (a.north) to (b.north);
  \end{tikzpicture}
\end{equation}
\vspace{1em}

\begin{lemma}
Let $\pi$ be an optimal policy for robust example-based control (Eq.~\ref{eq:robust-obj}). Assume that $\pi$ visits each success example with probability proportional to how often that state occurs as a success example:
\begin{equation}
    \exists \; 0 < c \le 1 \; \text{ s.t. }\; p^\pi(\fs = s) = c \cdot p(\cs = s \mid \e = 1) \; \forall s \in \gS. \label{eq:lemma-assumption}
\end{equation}
Then $\pi$ is a fixed point for iterated RCE.
\end{lemma}
\begin{proof}
If $p^\pi(\fs) = p(\cs \mid \e = 1)$, then Eq.~\ref{eq:bayes} tells us that $p(\e = 1 \mid \cs)$ is a constant for all states. In this setting, all policies are optimal under our objective (Def.~\ref{def:event-based-control}), so the current policy $\pi$ is also optimal. Thus, any policy satisfying $p^\pi(\fs) = c \cdot p(\cs \mid \e = 1)$ is a fixed point of iterated RCE.
\end{proof}
When we are given a discrete set of success examples, the assumption in Eq.~\ref{eq:lemma-assumption} is equivalent to assuming that the optimal robust policy spends an equal fraction of its time at each success example.

\begin{lemma}
Assume there exists a policy $\pi$ such that $p^\pi(\fs) = p(\cs \mid \e = 1)$. Then the solution to robust example-based control is a fixed point of iterated RCE.
\end{lemma}
\begin{proof}
As shown in Lemma~\ref{lemma:robust}, the robust example-based control corresponds to minimizing the squared Hellinger distance, an $f$-divergence. All $f$-divergences are minimized when their arguments are equal (if feasible), so the solution to robust  example-based control is a policy $\pi$ satisfying $p^\pi(\fs) = p(\cs \mid \e = 1)$.

Now, we need to show that such a policy is a fixed point of iterated RCE. If $p^\pi(\fs) = p(\cs \mid \e = 1)$, then Eq.~\ref{eq:assumption} tells us that $p(\e = 1 \mid \cs)$ is a constant for all states. In this setting, all policies are optimal under our objective (Def.~\ref{def:event-based-control}), so the current policy $\pi$ is also optimal. Thus, any policy satisfying $p^\pi(\fs) = p(\cs \mid \e = 1)$ is a fixed point of iterated RCE.
\end{proof}
While this result offers some explanation for why iterated RCE might be robust, the result is quite weak. First, the assumption that $p^\pi(\fs) = p(\cs \mid \e = 1)$ for some policy $\pi$ is violated in most practical environments, as it would imply that the agent spends every timestep in states that are success examples. Second, this result only examines the fixed points of robust example-based control and iterated RCE, and does not guarantee that these methods will actually converge to this fixed point. We aim to lift these limitations in future work.

\section{Connection with Success Classifiers and Success Densities}
\label{appendix:connections}

In this section we show that objective function in prior methods for example-based control implicitly depend on where success examples came from.

\citet{nasiriany2020disco} learn a reward function by fitting a density model to success examples, $\hat{p}(s^*) \approx p(\cs \mid \e = 1)$. Using this learned density as a reward function is \emph{not} the same as our approach. Whereas our method implicitly corresponds to the reward function $r(\cs) = (1 - \gamma) p(\e = 1 \mid \cs)$, the reward function in~\citet{nasiriany2020disco} corresponds to $\exp r(\cs)  = p(\cs \mid \e = 1) = p(\e = 1 \mid \cs) p_U(\cs) / p(\e = 1)$. The additional $p_U(\cs)$ term in the (exponentiated) Disco RL reward function biases the policy towards visiting states with high density under the $p_U(\cs)$ marginal distribution, \emph{regardless of whether those states were actually labeled as success examples}. 

In contrast, VICE~\citep{fu2018variational} learns a classifier to distinguish success examples $s^* \sim p(\cs \mid \e = 1)$ from ``other'' states sampled from $q(\cs)$. The predicted probability ratio from the Bayes-optimal classifier is
\begin{equation*}
    \frac{p(\cs \mid \e = 1)}{q_E(\cs)} = \frac{p(\e = 1 \mid \cs) p_U(\cs)}{p(\e = 1) q(\cs)}.
\end{equation*}
This term depends on the user's state distribution $p_U(\cs)$, so without further assumptions, using this probability ratio as a reward function does not yield a policy that maximizes the future probability of success (Eq.~\ref{eq:future-prob}). We can recover example-based control if we make the additional assumption that $p_U(\cs) = q(\cs)$, an assumption not made explicit in prior work. Even with this additional assumption, VICE differs from RCE by requiring an additional classifier.

In summary, these prior approaches to example-based control rely on auxiliary function approximators and do not solve the example-based control problem (Def.~\ref{def:event-based-control}) without additional assumptions. In contrast, our approach is simpler and is guaranteed to yield the optimal policy.

\section{Experimental Details}
\label{appendix:details}

\subsection{Implementation}
This method is straightforward to implement on top of any actor-critic method, such as SAC or TD3, only requiring a few lines of code to be added. The only changes necessary are to (1) sample success in the train step and (2) swap the standard Bellman loss with that in Eq.~\ref{eq:ce}. Unless otherwise noted, all experiments were run with 5 random seeds. When plots in the paper combine the results across seeds, this is done by taking the average.

We based our implementation of RCE off the SAC implementation from TF-Agents~\citep{TFAgents}. We modified the critic loss to remove the entropy term and use our loss (Eq.~\ref{eq:ce}) in place of the Bellman loss, modified the TD targets to use n-step returns with $n = 10$, and modified the actor loss to use a fixed entropy coefficient of $\alpha = 1e-4$. To implement n-step returns, we replaced the target value of $y = \frac{\gamma w^{(t)}}{\gamma w^{(t)} + 1}$ with $y = \frac{1}{2} \left(\frac{\gamma w^{(t)}}{\gamma w^{(t)} + 1} + \frac{\gamma^{10} w^{(t+10)}}{\gamma^{10} w^{(t+10)} + 1} \right).$ We emphasize that implementing RCE therefore only required modifying a few lines of code. Unless otherwise noted, all hyperparameters were taken from that implementation (version v0.6.0). For the image-based experiments, we used the same architecture from CURL~\citep{srinivas2020curl}, sharing the image encoder between the actor and classifier. Initial experiments found that using random weights for the image encoder worked at least as well as actually training the weights of the image encoder, while also making training substantially faster. We therefore used a random (untrained) image encoder in all experiments and baselines on the image-based tasks. We used the same hyperparameters for all environments, with the following exception:
\begin{itemize}
    \item \texttt{sawyer\_bin\_picking}: The SAC implementation in TF-Agents learns two independent Q values and takes the minimum of them when computing the TD target and when updating the policy, as suggested by~\citet{fujimoto2018addressing}. However, for this environment we found that the classifier's predictions were too small, and therefore modified the actor and critic losses to instead take the maximum over the predictions from the two classifiers. We found this substantially improved performance on this task.
\end{itemize}

\subsection{Baselines.} We implemented the SQIL baseline as an ablation to our method. The two differences are (1) using the standard Bellman loss instead of our loss (Eq.~\ref{eq:ce}) and (2) not using n-step returns. We implemented all other baselines on top of the official DAC repository~\citep{kostrikov2018discriminator}. For fair comparison with our method, the classifier was only conditioned on the state, not the action. These baselines differed based on the following hyperparameters:
\begin{table}[H]
    \centering
    \begin{tabular}{c|c|c|c}
        Method name & gradient penalty & discriminator loss & absorbing state wrapper \\ \hline
        DAC & 10 & \texttt{cross\_entropy} & yes\\
        AIRL & 0 & \texttt{cross\_entropy} & no \\
        ORIL-PU & 0 & \texttt{positive\_unlabeled} & no \\
        PURL & 0 & \texttt{positive\_unlabeled\_margin} & no \\
    \end{tabular}
    \vspace{0.5em}
    \caption{Classifier-based baselines}
\end{table}
The done bit wrapper is the idea discussed in DAC~\citep[Sec. 4.2]{kostrikov2018discriminator}. The \texttt{cross\_entropy} loss is the standard cross entropy loss. The \texttt{positive\_unlabeled} loss is the standard positive-unlabeled loss~\citep{elkan2008learning}; following~\citet[Appendix B, Eq. 5]{zolna2020offline} we used $\eta = 0.5$. The \texttt{positive\_unlabeled\_margin} loss is taken from~\citet[Eq. 9]{xu2019positive}, as are the hyperparameters of $\eta = 0.5, \beta = 0.0$.

\subsection{Environments}
\label{appendix:environments}

For both the state-based experiments and the image-based experiments, we choose tasks using two design criteria: (1) we wanted to include a breadth of tasks, and (2) we wanted to include the most difficult tasks that could be learned by either our method or any of the baselines. For the image-based tasks, we were unable to get the prior methods to learn tasks more complex than those included in Fig.~\ref{fig:reach-generalization}.

Our experiments used benchmark manipulation tasks from Metaworld (Jan. 31, 2020)~\citep{yu2020meta} and D4RL (Aug. 26, 2020)~\citep{fu2020d4rl, rajeswaran18learning}. Unless otherwise mentioned, we used the default parameters for the environments. The reward functions for some of the Sawyer tasks contain additional reward shaping terms. As we describe below, we remote these shaping terms to more directly evaluate success. 
\begin{enumerate}
    \item \texttt{sawyer\_drawer\_open}: This task is based on the \texttt{SawyerDrawerOpenEnv} from~\citet{yu2020meta}. We generated success examples by sampling the drawer Y coordinate from \texttt{unif}(low=-0.25, high=-0.15) and moving the gripper to be touching the drawer handle. Episodes were 151 steps long. For evaluation, we used the net distance that the robot opened the drawer (displacement along the negative Y axis) as the reward.
    \item \texttt{sawyer\_push}: This task is based on the \texttt{SawyerReachPushPickPlaceEnv} from~\citet{yu2020meta}, using \texttt{task\_type = `push'}. We generated success examples by sampling the puck XY position from \texttt{unif(low=(0.05, 0.8), high=(0.15, 0.9))} and moving the gripper to be touching the puck. Episodes were 151 steps long. For evaluation, we used the net distance that the puck traveled towards the goal as the reward.
    \item \texttt{sawyer\_lift}: This task is based on the \texttt{SawyerReachPushPickPlaceEnv} from~\citet{yu2020meta}, using \texttt{task\_type = `reach'}. We generated success examples by sampling the puck Z coordinate from \texttt{unif}(low=0.08, high=0.12) and moving the gripper such that the puck was inside the gripper. Episodes were 151 steps long. For evaluation, we used the Z coordinate of the puck at the final time step.
    \item \texttt{sawyer\_box\_close}: This task is based on the \texttt{SawyerBoxCloseEnv} from~\citet{yu2020meta}. We generated success examples by positioning the lid on top of the box and positioning the gripper so that the lid handle was inside the gripper. Episodes were 151 steps long. For evaluation, we used the net distance that the lid traveled along the XY plane towards the center of the box.
    \item \texttt{sawyer\_bin\_picking}: This task is based on the \texttt{SawyerBinPickingEnv} from~\citet{yu2020meta}. We generated success examples by randomly positioning the object in the target bin, sampling the XY coordinate from \texttt{unif(low=(0.06, 0.64), high=(0.18, 0.76))} and positioning the gripper so that the object was inside the gripper. Episodes were 151 steps long. For evaluation, we used the net distance that the object traveled along the XY plane towards the center of the target bin.
    \item \texttt{door-human-v0}, \texttt{hammer-human-v0}: These tasks are described in~\citep{rajeswaran18learning}. We generated success examples by taking a random subset of the last 50 observations from each trajectory of demonstrations provided for these tasks in~\citep{fu2020d4rl}. We used the default episode length of 200 for these tasks. For evaluation, we used the default reward functions for these tasks (which are a few orders of magnitude larger than the net-distance rewards used for evaluating the sawyer tasks.). 
    \item \texttt{sawyer\_reach\_random\_position\_image}, \texttt{sawyer\_reach\_random\_size\_image}: These image-based reaching tasks are based on the \texttt{SawyerReachPushPickPlaceEnv} task from~\citet{yu2020meta}, using \texttt{task\_type = `reach'}. We generated success examples by positioning the gripper at the object. Episodes were 51 steps long. For evaluation, we used the net distance traveled towards the object as the reward. The \texttt{sawyer\_reach\_random\_position\_image} modified the environment to initialize each episode by randomly sampling the object position from \texttt{unif(low=(-0.02, 0.58), high=(0.02, 0.62))}. The \texttt{sawyer\_reach\_random\_size\_image} modified the environment to initialize each episode by randomly resizing the object by sampling a radius and half-height from \texttt{unif(low=(0.1, 0.005), high=(0.05, 0.045))}.
    \item \texttt{sawyer\_clear\_image}: This task is based on the \texttt{SawyerReachPushPickPlaceEnv} from~\citet{yu2020meta}, using \texttt{task\_type = `push'}. We generated success examples by simply deleting the object from the scene and setting the gripper to a random position drawn from \texttt{unif(low=(-0.2, 0.4, 0.02), high=(0.2, 0.8, 0.3))}. For evaluation, we used a reward of 1 if the object was out of sight, as defined by having a Y coordinate less than 0.41 or greater than 0.98.
\end{enumerate}

\section{Ablation Experiments and Visualization}
\label{appendix:ablation}

\begin{figure}[!t]
    \centering
    \vspace{-1em}
    \begin{subfigure}[b]{0.49\textwidth}
        \includegraphics[width=\linewidth]{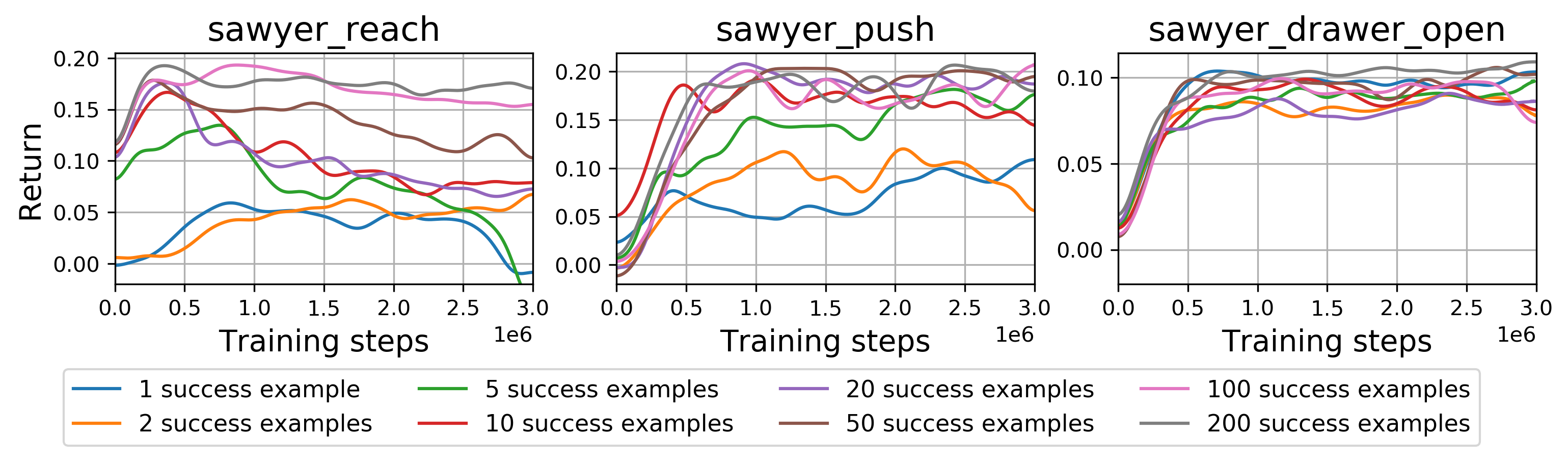}
        \caption{Number of success examples}
        \label{fig:num-expert-obs-ablation}
    \end{subfigure}
    \begin{subfigure}[b]{0.49\textwidth}
        \includegraphics[width=\linewidth]{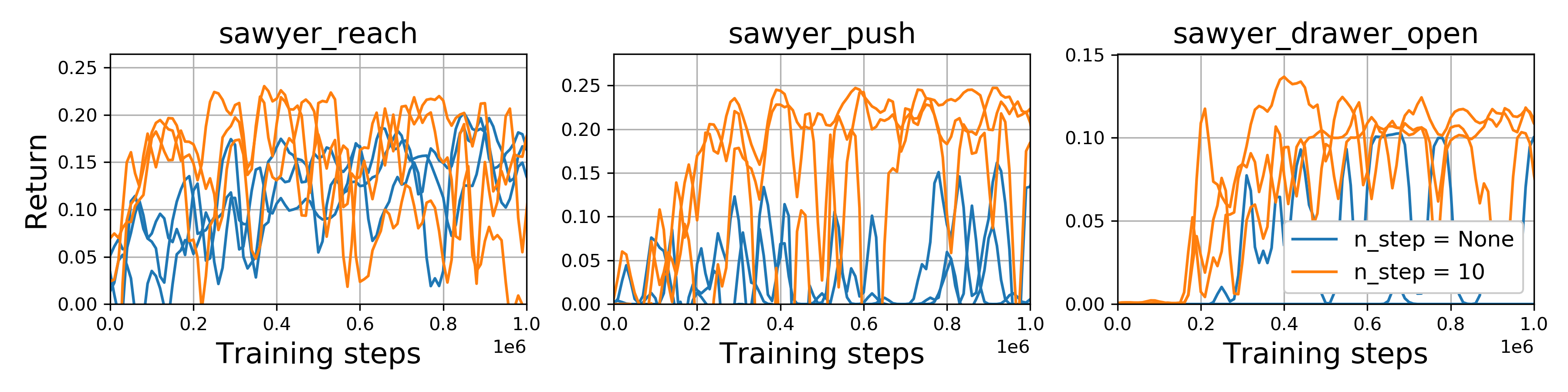}
        \vspace{0.1em}
        \caption{N-step returns}
        \label{fig:n-step-ablation}
    \end{subfigure}
    \begin{subfigure}[b]{0.49\textwidth}
        \includegraphics[width=\linewidth]{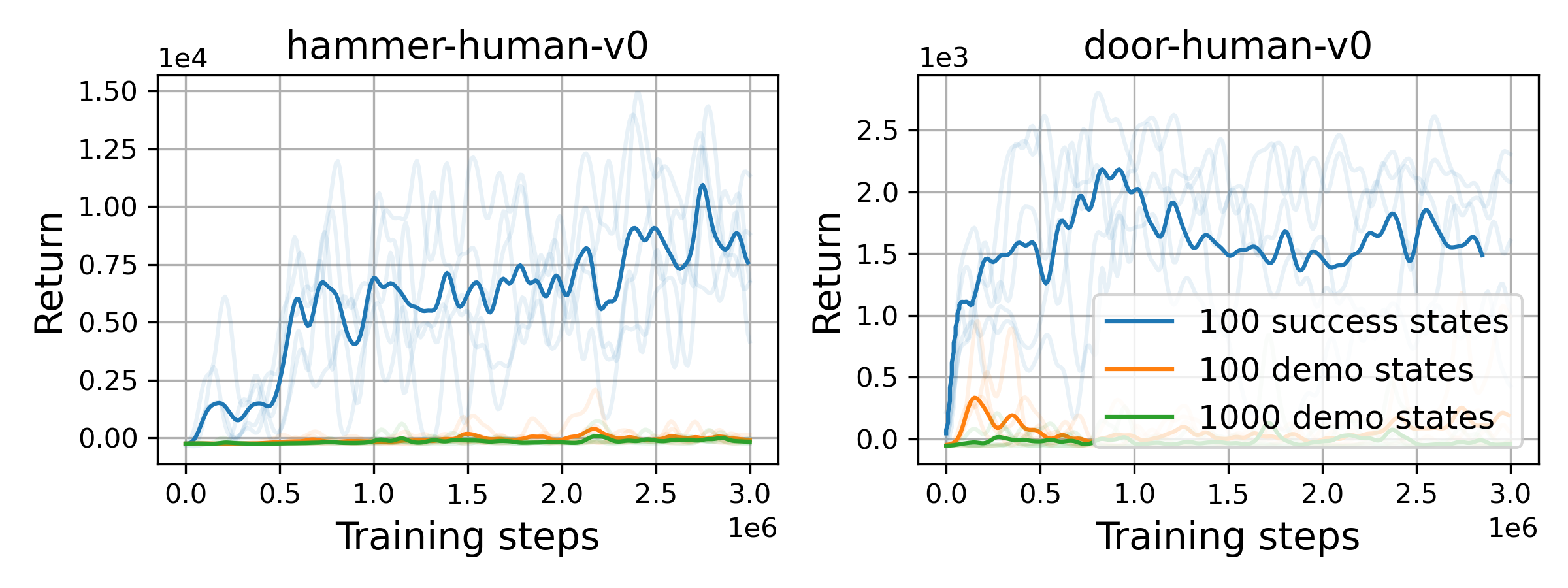}
        \caption{Using all states from demonstrations.} \label{fig:demos}
    \end{subfigure}
    \begin{subfigure}[b]{0.49\textwidth}
        \includegraphics[width=\linewidth]{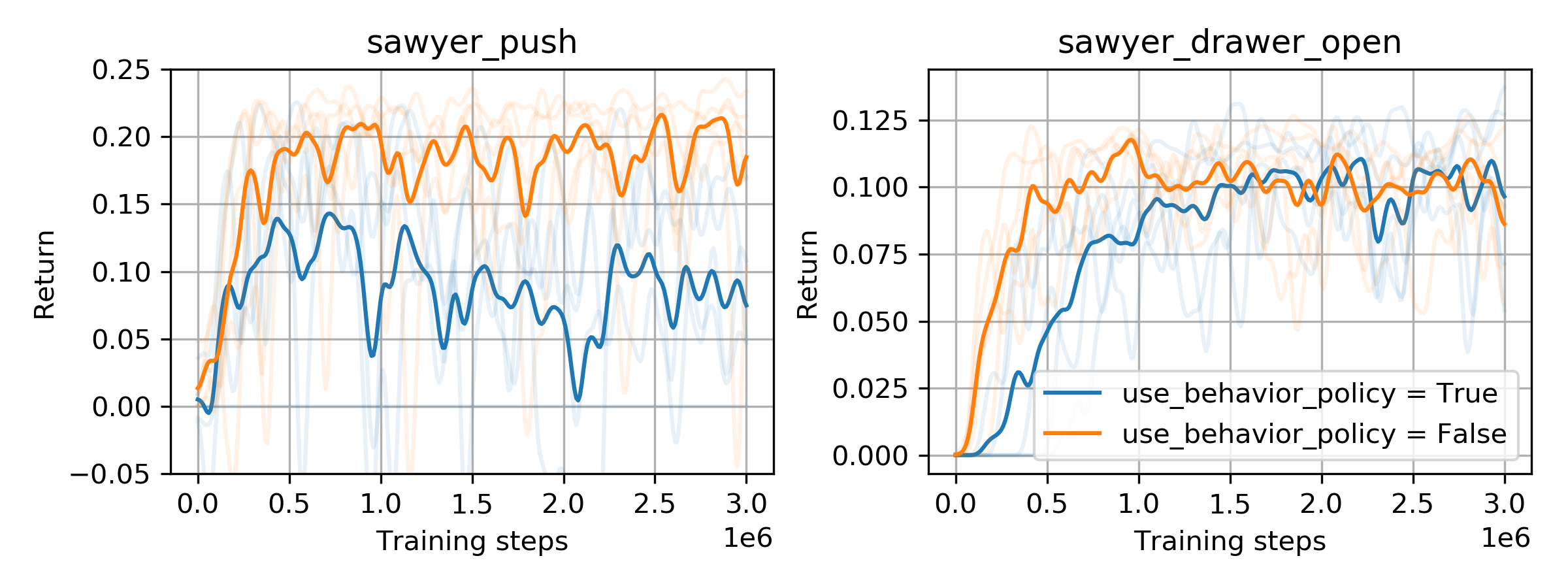}
        \caption{Actions for success examples}
        \label{fig:action-ablation}
    \end{subfigure}
    \begin{subfigure}[b]{0.49\textwidth}
        \includegraphics[width=\linewidth]{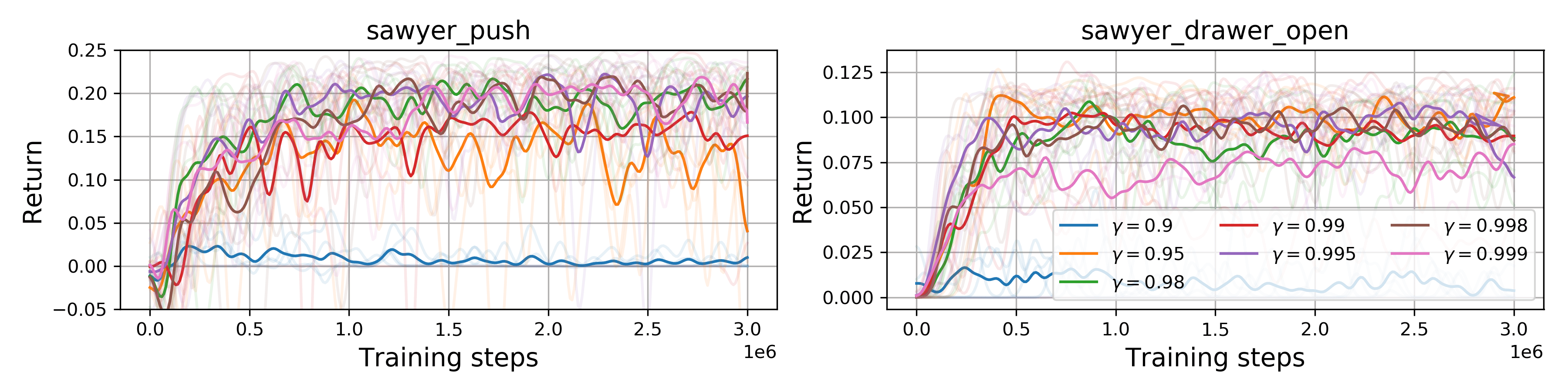}
        \caption{Ablation of discount factor.} \label{fig:gamma}
    \end{subfigure}
    \begin{subfigure}[b]{0.49\textwidth}
        \includegraphics[width=\linewidth]{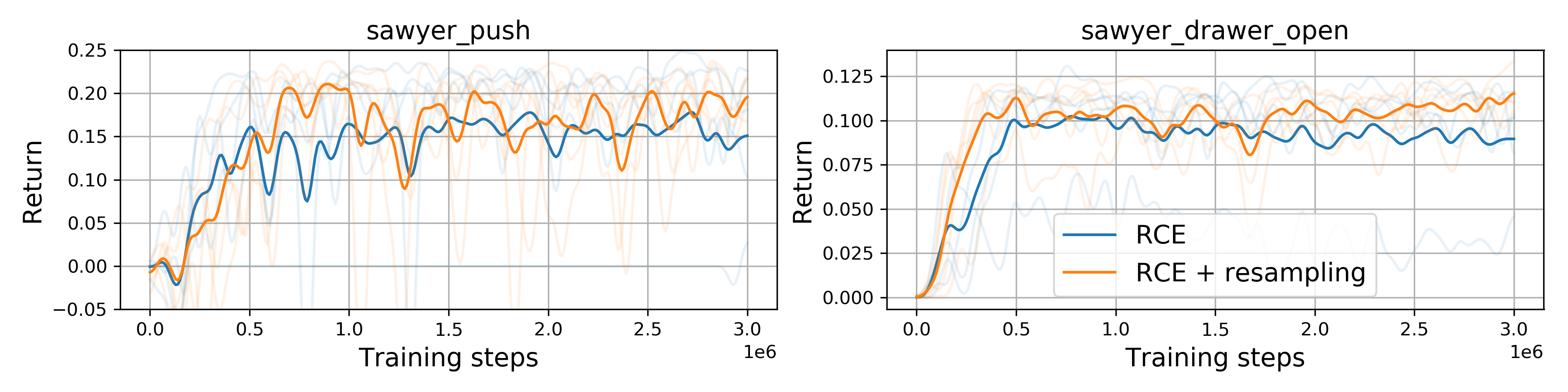}
        \caption{Sampling new success examples.} \label{fig:resampling}
    \end{subfigure}
    \begin{subfigure}[b]{0.49\textwidth}
        \includegraphics[width=\linewidth]{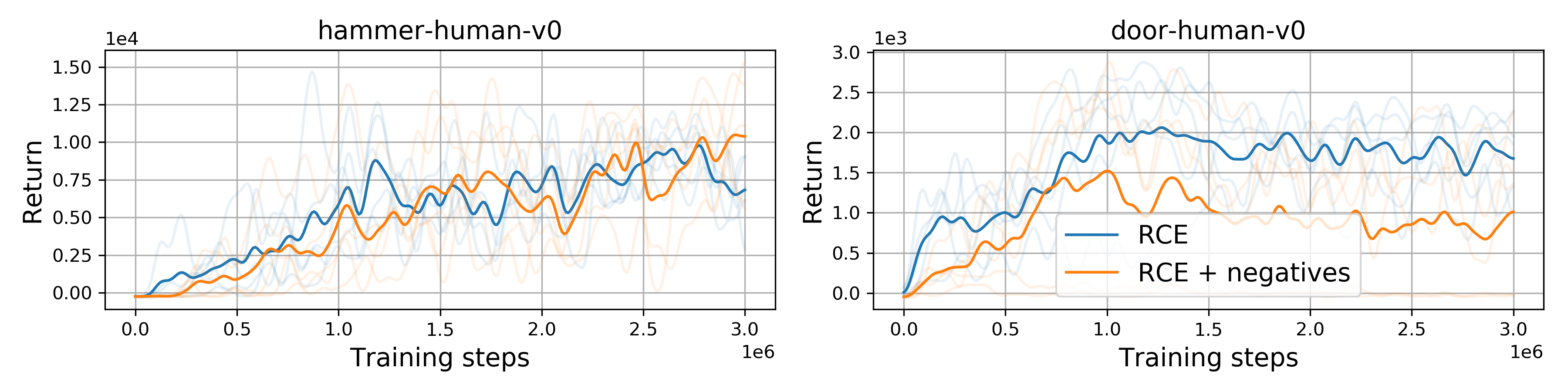}
        \caption{Using negative examples.} \label{fig:negatives}
    \end{subfigure} 
    \caption{\textbf{Ablation Experiments} \label{fig:ablation}}
\end{figure}

We ran a number of ablation experiments to further understand our method.

\paragraph{Number of success examples.} We studied performance when we used fewer than 200 success examples. As shown in Fig.~\ref{fig:num-expert-obs-ablation}, we can achieve comparable performance using 100 examples for the  \texttt{sawyer\_reach} task, 20 examples for the \texttt{sawyer\_push} task, and just 1 example for the \texttt{sawyer\_drawer\_open} task. We ran each experiment with 5 random seeds, but only plot the mean across seeds to decrease clutter.

\paragraph{N-step returns.}
An important implementation detail of RCE is using n-step returns. Fig.~\ref{fig:n-step-ablation} shows that using n-step returns substantially improved performance on the \texttt{sawyer\_push} and \texttt{sawyer\_drawer\_open} tasks, and had a smaller positive effect on the \texttt{sawyer\_reach} task. We used n-step returns with $n = 10$ for all other experiments with our method.

\paragraph{Applying RCE to expert trajectories.}
RCE is intended to be applied to a set of success examples. Applying this same method to states sampled randomly from expert demonstrations does not work, as shown in Fig.~\ref{fig:demos}. Not only is requiring intermediate states from demonstrations more onerous for users, but also RCE performs worse in this setting. Interestingly, we found the same trend for the SQIL baseline, which (unlike our method) was designed to use states from demonstrations, not just success states. 
Thus, while SQIL is designed as an imitation learning method, it may actually be solving a version of example based control.

\paragraph{Sampling actions for success examples.}
We studied the approximation made in Sec.~\ref{sec:method}, where we sampled actions for the success examples using our current policy instead of the historical average policy. We fit a new behavior policy to our replay buffer using behavior cloning and used this behavior policy to sample actions for the success examples. Fig.~\ref{fig:action-ablation} shows that simply using our current policy instead of this behavior policy worked at least as well, while also being simpler to implement.

\paragraph{New success examples during training.}
RCE is intended to be applied in the setting where a set of success examples is provided before training. Nonetheless, we can apply RCE to a set of success examples that is updated throughout training. We implemented this variant of RCE, sampling the success set anew every 100k iterations. Fig.~\ref{fig:resampling} shows that sampling new success examples online during training does not affect the performance of RCE. 

\paragraph{Negative examples.}
In this ablation experiment, we modified RCE by training the classifier to predict 0 for a set of ``negative'' states. We implemented this variant of RCE, using samples from the initial state (the task is never solved at the initial state) as negatives. Fig.~\ref{fig:negatives} shows that this variant performs about the same on the hammer task, but worse on the door task.

\paragraph{Discount factor, $\gamma$.}
As shown in Fig.~\ref{fig:gamma}, the performance of RCE is indistinguishable for a wide range of discount factors (0.95 to 0.998). RCE does perform worse for very small (0.9) and very large (0.999) discount factors. Note that we did not tune this parameter for the experiments in the main paper, and used a value of $\gamma = 0.99$ for all experiments there.

\begin{figure}[ht]
    \centering
    \begin{minipage}{0.5\textwidth}
       \centering
        \begin{subfigure}[c]{0.5\linewidth}
        \includegraphics[width=\linewidth]{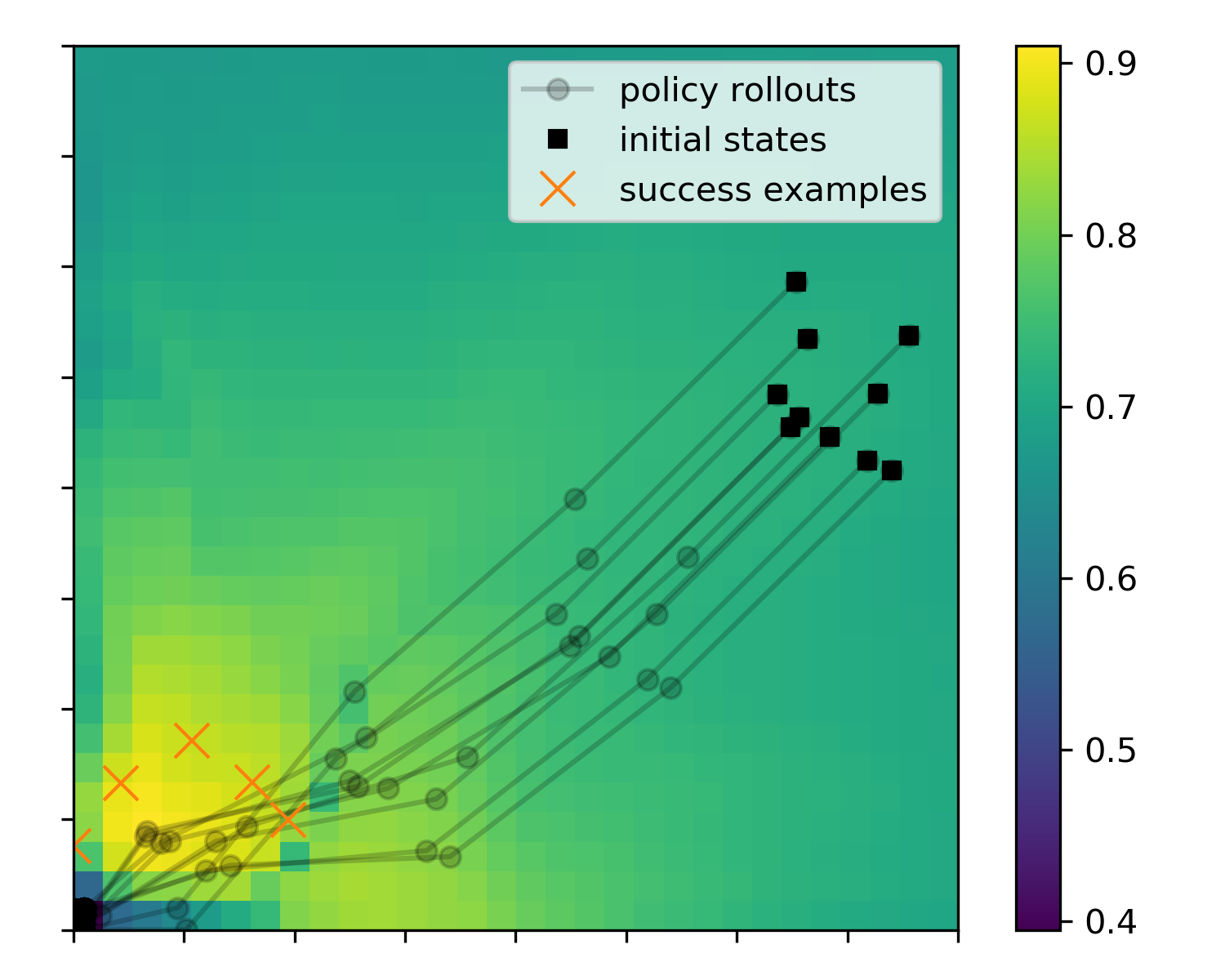}
        \caption{20k iterations}
        \end{subfigure}%
        ~
        \begin{subfigure}[c]{0.5\linewidth}
        \includegraphics[width=\linewidth]{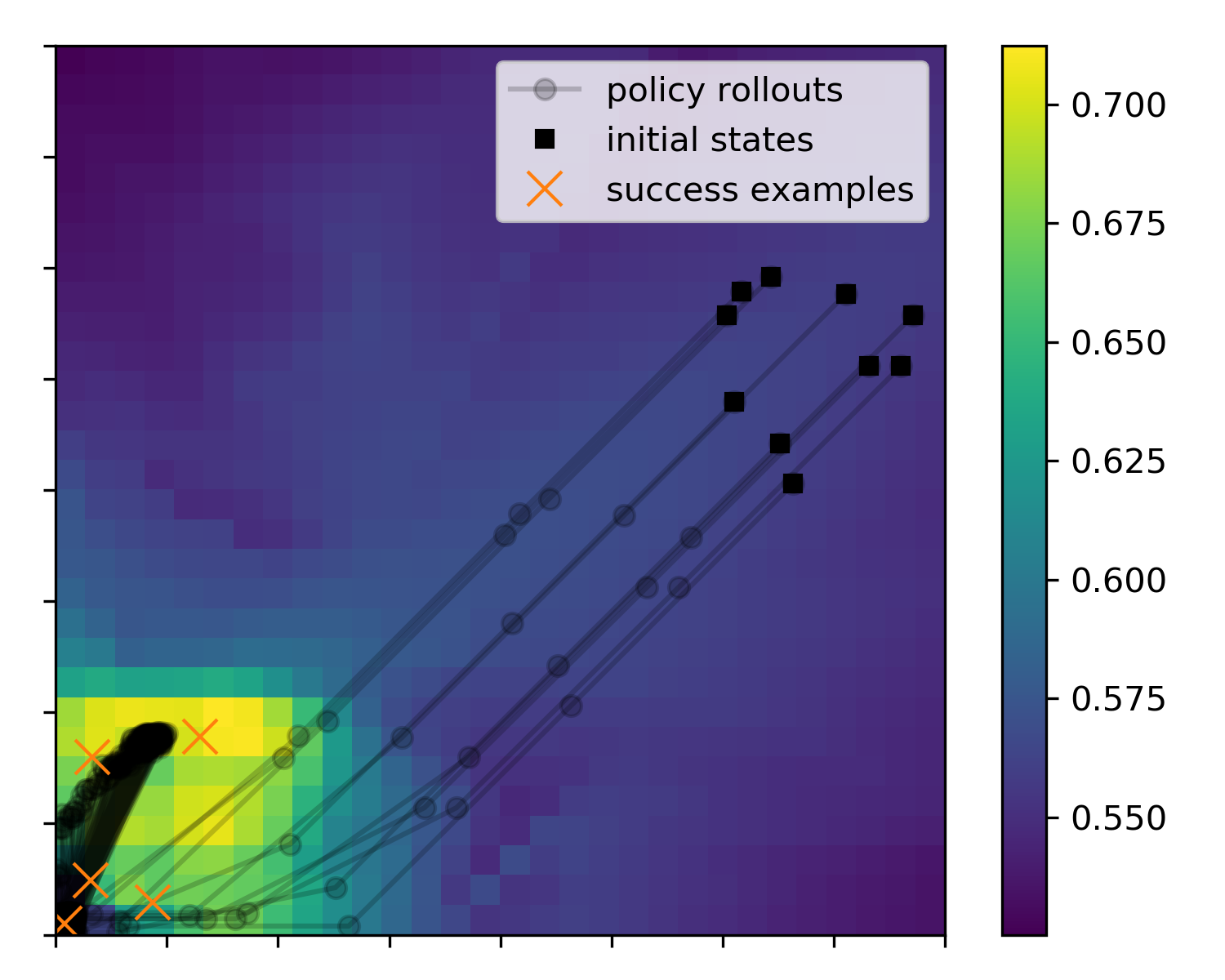}
        \caption{80k iterations}
        \end{subfigure}
        \caption{\footnotesize Visualizing the classifier's predictions and policy rollouts throughout training. To avoid visual clutter, we plot a random subset of the success examples.}
        \label{fig:classifier-viz}
    \end{minipage}
    \hfill
    \begin{minipage}{0.45\textwidth}
    \centering
        \includegraphics[width=\linewidth]{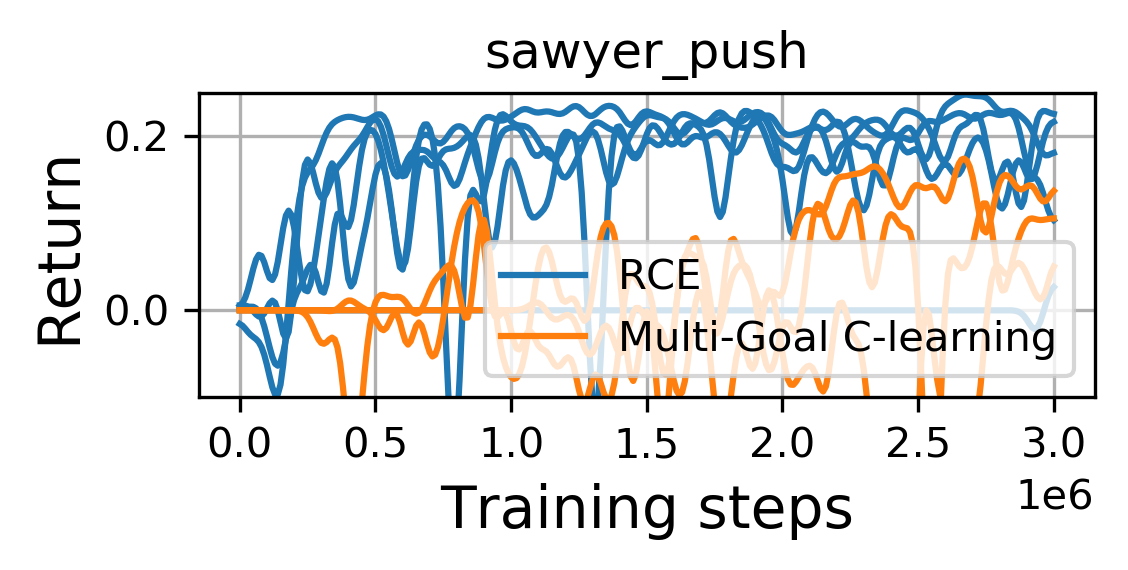}
        \caption{\footnotesize We compare RCE against an alternative approach based on goal-conditioned RL.} \label{fig:cim}
    \end{minipage}
\end{figure}

\subsection{Visualizing Learning Dynamics}
To better illustrate the mechanics of our method, we ran RCE on a simple 2D navigation task. We then visualized the classifier's predictions throughout training. As shown in Fig.~\ref{fig:classifier-viz}, the classifier's predictions are initially mostly uniform across all states, but later converse to predict much larger values near the success examples and lower values elsewhere.

\subsection{Comparison with C-Learning}
C-learning~\citep{eysenbach2020c} solves a different problem than RCE and thus cannot be directly applied to this setting. Specifically, C-learning learns a goal-conditioned policy from a dataset of transitions, whereas RCE learns a single policy (not conditioned on goals) from a dataset of transitions and another dataset of success examples. Nonetheless, we initially tried to solve example-based control problems using a version of C-learning modified to use success examples. However, we found that this method performed quite poorly. For example, on the sawyer-push task, RCE achieves an average return of 0.2 (see Fig.~\ref{fig:cim}) whereas this variant of C-learning achieved a return of about 0.07 (65\% worse). The poor performance of this variant of C-learning motivated us to design a new method from the ground-up, specifically designed to solve the example-based control problem. We will include a plot of these initial experiments in the final paper.

\section{Failed Experiments}

This section describes a number of experiments that we tried that did not work particularly well. We did not include any of these ideas in our final method.
\begin{enumerate}
    \item \textbf{Training the image encoder.} For the image-based experiments, we found that \emph{not} training the image encoder and just using random weights often worked better than actually learning these weights.
    \item \textbf{Alternative parametrization of the classifier.} By default, our classifier predicts values in the range $C_\theta \in [0, 1]$. However, Eq.~\ref{eq:posterior} says that the classifier's predicted probability ratio corresponds to
    \begin{equation*}
    p^\pi(\fe = 1 \mid \cs, \ca) = \frac{C_\theta^{(t)}}{1 - C_\theta^{(t)}} \in [0, 1].
    \end{equation*}
    We therefore expect that a learned classifier should only make predictions in the range $[0, 0.5]$. We experimented by parametrizing the classifier to only output values in $[0, 0.5]$, but found that this led to no noticeable change in performance.
    \item \textbf{Classifier regularization.} We experimented with regularizing our classifier with input noise, gradient penalties, mixup, weight decay and label smoothing. These techniques yield improvements that were generally small, so we opted to not include any regularization in our final method to keep it simple.
    \item \textbf{Regularizing the \emph{implicit} reward.} We can invert our Bellman equation (Eq.~\ref{eq:bellman}) to extract the reward function that has been implicitly learned by our method:
    \begin{equation*}
        r(\cs, \ca) = \frac{C_\theta^{(t)}}{1 - C_\theta^{(t)}} - \gamma \frac{C_\theta^{(t+1)}}{1 - C_\theta^{(t+1)}}.
    \end{equation*}
    We experimented with regularizing this implicit reward function to be close to zero or within $[0, 1]$, but saw no benefit.
    \item \textbf{Normalizing $w$.} We tried normalizing the TD targets $w$ (Eq.~\ref{eq:w}) to have range $[0, 1]$ by applying a linear transformation (subtracting the min and dividing by the max-min), but found that this substantially hurt performance.
    \item \textbf{Hyperparameter robustness.} We found that RCE was relatively robust to the learning rate, the discount factor, the Polyak averaging term for target networks, the batch size, the classifier weight initialization, and the activation function.
    \item \textbf{Imbalanced batch sizes.} Recall that our objective (Eq.~\ref{eq:ce}) uses a very small weight of $(1 - \gamma)$ for the first term. This small weight means that our effective batch size is relatively small. We experimented with changing the composition of each batch to include fewer success examples and more random examples, updating the coefficients for the loss terms such that the overall loss remained the same. While this does increase the effective batch size, we found it had little effect on performance (perhaps because our method is fairly robust to batch size, as noted above).
    \item \textbf{Optimistic initialization.} We attempted to encourage better exploration by initializing the final layer bias of the classifier so that, at initialization, the classifier would predict $C_\theta = 1$. We saw no improvements from this idea.
    \item \textbf{More challenging image-based tasks.} We found that RCE (and all the baselines) failed on image-based versions of the tasks in Fig.~\ref{fig:environments}. Our hypothesis is that this setting is challenging for RCE for the same reason that image-based tasks are challenging for off-policy RL algorithms: it is challenging to learn Q functions via temporal different learning on high-dimensional inputs. In future work, we aim to study the use of explicit representation learning to assist RCE in learning complex image-based tasks.

\end{enumerate}

\end{document}